%% file: main.tex
\documentclass[journal]{IEEEtran} 

\IEEEoverridecommandlockouts                              

\usepackage{cite}
\usepackage{amsmath,amssymb,amsfonts,enumitem}
\usepackage{algorithmic}
\usepackage{graphicx}
\usepackage{textcomp}
\usepackage{xcolor}
\usepackage{siunitx} 
\usepackage{amsthm}
\usepackage{mathtools}
\usepackage{cleveref}
\usepackage{placeins}
\usepackage{dblfloatfix}
\usepackage[acronym]{glossaries}
\usepackage{tikz}
\usepackage[normalem]{ulem}

\graphicspath{{./pictures/}}

\newcommand{\T}[1]{\boldsymbol{\mathbf{#1}}}
\newcommand{\inv}[1]{#1^{-1}}
\newcommand{\transp}[1]{#1^{T}}

\newcommand{\q}{\T{q}}
\newcommand{\dq}{\dot{\T{q}}}
\newcommand{\ddq}{\ddot{\T{q}}}

\newcommand{\x}{\T{x}}
\newcommand{\dx}{\dot{\T{x}}}
\newcommand{\ddx}{\ddot{\T{x}}}

\newcommand{\dqminus}{\dq^-}
\newcommand{\dqplus}{\dq^+}
\newcommand{\dxminus}{\dx^-}
\newcommand{\dxplus}{\dx^+}

\newcommand{\imp}{_\ast}
\newcommand{\iv}{_\textup{iv}}

\newcommand{\J}{\T{J}}
\newcommand{\M}{\T{M}}
\newcommand{\A}{\T{A}}

\newcommand{\Real}{\mathbb{R}}

\newcommand{\Pma}{\T{P}_{q}}
\newcommand{\elMap}{\T{Q}}
\newcommand{\Px}{\T{P}_{x}}
\newcommand{\taskMap}{\T{X}}
\newcommand{\range}{\textup{range}}

\newcommand{\Axfr}{\T{A}_{x,t}}
\newcommand{\lfr}{\T{\lambda}_t}
\newcommand{\Lfr}{\T{\Lambda}_t}
\newcommand{\Atotal}{\bar{\T{A}}}
\newcommand{\Axtotal}{\bar{\T{A}}_{x}}
\newcommand{\ltotal}{\bar{\T{\lambda}}}

\newcommand{\eq}{_{c}}
\newcommand{\Peq}{\left(\T I - \T P\eq^T\right)}

\newcommand{\Ml}{\T{M}_l}
\newcommand{\Mm}{\T{M}_m}
\newcommand{\ql}{\q_l}

\newcommand{\qm}{\q_m}

\newcommand{\pseudoInv}[2]{\T{#1}^{#2^\dagger}}
\newcommand{\Ax}{\A_x}

\newcommand{\Jmp}[1]{\inv{\M}\transp{\T{#1}}\inv{\left(\T{#1}\inv{\M}\transp{\T{#1}}\right)}}

\newcommand{\note}[1]{\textcolor{black}{#1}}

\newtheorem{theorem}{Theorem}
\newtheorem{lemma}{Lemma}
\newtheorem{assumption}{Assumption}
\newtheorem{corollary}{Corollary}

\setacronymstyle{long-short}
\newacronym{CoM}{CoM}{center of mass}

\newcommand*\circled[1]{\tikz[baseline=(char.base)]{
		\node[shape=circle,draw,inner sep=1pt] (char) {\scriptsize{#1}};}}

\usepackage{setspace}
\definecolor{sidebox_bg}{rgb}{1,1,0.4}

\title{The Nonsmooth Impact Direction (NSID) \note{of} Robotic \note{Systems} 
	\thanks{This project has received funding from the European Research Council (ERC) under the European Union’s Horizon 2020 research and innovation programme (Grant agreement No. 101248099).}	
}
\author{
		Annika Kirner and Christian Ott \thanks{Annika Kirner and Christian Ott are with the Automation and Control Institute, TU Wien, 1040
		Vienna, Austria (e-mail: annika.kirner@tuwien.ac.at; christian.ott@tuwien.ac.at). Christian Ott is also with the Institute of Robotics and Mechatronics, German Aerospace Center (DLR), 82234 Wessling, Germany.}}

\begin{document}

\maketitle
\thispagestyle{empty}
\pagestyle{empty}

\begin{abstract}
Collisions of rigid-link robots and rigid environments are often modeled as instantaneous events. Under this idealization, the impact forces become impulsive and the system velocities nonsmooth. In this work, we systematically analyze pre- and post-impact velocities focusing on what we refer to as the "nonsmooth impact direction" (NSID). \note{We show that it is a characteristic direction of a robotic impact and largely independent of contact properties.
} The results are directly applicable to large classes of backdrivable robotic systems with rigid links. \note{ We address particularities of systems with non-elastic and flexible joints, unconstrained as well as constrained systems. Further, we show that the approach direction w.r.t the NSID sets the direction of the impulsive force in frictional, inelastic impacts. The comprehensive theoretical analysis of this work supported by an experimental validation may serve as a foundation for future planning and control algorithms for various robotic impact applications. These can include humanoid locomotion on a slippery surface or repetitive hammering.}
\end{abstract}


\input{introduction}
\input{background}
\input{model2}
\input{interpretations}
\input{friction_analysis_NEW2}
\input{extended_models_V2}

\input{experiments}

\appendix
\subsection{\note{Projectors}}\label{app:projectors}
\note{A matrix $ \T P \in \Real^{a\times a} $ is a projector, iff it is idempotent, i.e. if it holds $ \T P^2 = \T P $. By this definition, all eigenvalues of a projector can be shown to be either $ 0 $ or $ 1 $. 
In this work, mainly projectors expressed in the form \eqref{eq:projector}, i.e. $ \T P = \pseudoInv{U}{\T W}\T{U} $ 
are encountered. The range of the projector \eqref{eq:projector} and thus the eigenspace corresponding to the $ b $-fold eigenvalue  of $ 1 $ can be identified to be $ {\range(\T P) = \range(\inv{\T{W}}\T{U}^T)} $. For the $ {(a-b)} $-dimensional kernel it holds $ {\ker(\T P) = \ker(\T U)} $, where $ {\ker({\T U}) = \range(\T U^T)^\perp}$. Note that $ \ker(\T P) $ can be interpreted as the direction of the projection onto the $ (a-b) $-dimensional subspace of $ \Real^a $ defined by $ \range(\inv{\T{W}}\T{U}^T) $. The eigenspaces of $ \T P $ are summarized in \Cref{tab:eig P}. Besides $ \T P $, also $ \T P_2 = \T{I} - \T P $ is a projector. It can be intuitively understood that it holds $ {\range(\T P_2) = \ker(\T P)} $ and $ {\ker(\T P_2) = \range(\T P)} $.
\begin{table}
	\centering
	\setlength\extrarowheight{3pt}
	\begin{tabular}{ c | c c }
		\hline
		eigenvalue & dimension of eigenspace & eigenspace \\ 
		\hline
		0 & a-b &$ \ker(\T U) $\\  
		1 & b & $ \range(\inv{\T W}\T U^T) $ \\
		\hline  
	\end{tabular}
	\vspace{1mm}
	\caption{Eigenspaces of a projector $ \T P $ defined as $ \T P = \pseudoInv{U}{\T W}\T{U} $ with matrix $ {\T{U}\in\Real^{a\times b}} $ of rank $ b $}
	\label{tab:eig P}
\end{table}}

\subsection{\note{Proof of \Cref{thm:constrained model}}}\label{app:constrained_model}
\note{Consider the constrained robotic system \eqref{eq:RD rigid} with \eqref{eq:eq constraint tau contact} as described in Sec.~\ref{sec:constraint}. In order for the equality constraint to remain fulfilled, it needs to hold $ \T\phi\eq = \T 0 $, $ \A\eq\dq = \T 0 $, and $ \dot{\A}\eq \dq + \A\eq\ddq = \T 0 $. With this 
we have 
\begin{align}\label{eq:lambda equality}
	\nonumber
	\T\lambda\eq = &- \left(\pseudoInv{A}{\M}\eq\right)^T \left(\T\tau + \A^T\lambda - \T h(\q,\dq)\right)\\
	& - \inv{\left(\A\eq\inv{\M}\A\eq^T\right)}\dot{\A}\eq \dq
\end{align}
for the generalized equality constraint force.}

\note{Consider an impact such that the contact and configuration satisfy Assumptions~\ref{a:non singular basic}-\ref{a:constraint}.
Integrating the dynamics \eqref{eq:RD rigid} with \eqref{eq:eq constraint tau contact} over the infinitesimal impact duration (\Cref{a:instantaneous}), we obtain
\begin{align}\label{eq: impact equ constrained}
	\M\left(\dqplus - \dqminus\right) &= \A^T\Lambda + \A\eq^T\underbrace{\left(\lim_{\Delta t \rightarrow 0}\int_{t}^{t+\Delta t} \T\lambda\eq d\tau\right)}_{\T\Lambda\eq}.
\end{align}
The impulsive, generalized equality constraint force $ {\T\Lambda\eq \in \Real^p} $ therein, is obtained as 
\begin{align}\label{eq: impulsive constraint force}
	\T\Lambda\eq &= -\left(\pseudoInv{A}{\M}\eq\right)^T\A^T\Lambda 
\end{align}
by integration of \eqref{eq:lambda equality} over the impact duration.
Combining \eqref{eq: impact equ constrained} and \eqref{eq: impulsive constraint force} yields
\begin{align}\label{eq: impact equ constrained full}
	\M\left(\dqplus - \dqminus\right) &= \left(\T I - \underbrace{\A\eq^T\left(\pseudoInv{A}{\M}\eq\right)^T }_{\coloneq\T P\eq^T}\right)\A^T\Lambda.
\end{align}
We identify the dynamically consistent projector \cite{Khatib1995} $ \T I - \T P\eq^T $, mapping generalized torques on the constraint. Note that it holds
\begin{equation}\label{eq: MP}
	\inv{\M}\Peq = (\T I -\T P\eq) \inv{\M}.
\end{equation} The transposed projector $ \T I - \T P\eq $, projects generalized velocities on the null space of the equality constraint. All generalized velocities $ \dq $, and thus also $ \dqminus $ and $ \dqplus $, have to be compatible with the equality constraint at all times. Thus, in particular, it always holds $ \A\eq\dq = \T 0 $ and 
\begin{equation}\label{eq:compatible vel}
	\dq = (\T I - \T P\eq)\dq.
\end{equation}}

\note{With \eqref{eq: impact equ constrained full} and \Cref{a:restitution} we obtain the map \eqref{eq:impact map constrained} 
\begin{align}\label{eq:impact map constrained_app}
	\dqplus =  \left( \T I - \left(1+e\right)\T P_{q,c}\right)  \dqminus
\end{align}
utilized in \ref{sec:constraint}. 
The matrix $ \T P_{q,c} $ can be shown to be a projector and is first derived as
\begin{align}\label{eq:Pqc}
	\T P_{q,c} &= \inv{\M}\left(\T I - \T P\eq^T\right)\A^T\underbrace{\inv{\left(\A\inv{\M}\left(\T I - \T P\eq^T\right)\A^T\right)}}_{\circled 1}\A.
\end{align}
Note that the inverse marked with \circled{1} always exists due to \Cref{a:constraint}. 
Exploiting idempotence of $ {(\T I -\T P_c^T)} $, and \eqref{eq: MP} we can reformulate \eqref{eq:Pqc} in the form \eqref{eq:Pqc 2}, i.e.
\begin{align}
	\T P_{q,c} & = \T W_{q,c} \A^T\inv{\left(\A\T W_{q,c} \A^T\right)}\A,
\end{align}
with $ \T{W}_{q,c} = (\T I -\T P_c)\inv{\M}(\T I -\T P_c^T) \in \Real^{n\times n}$, which concludes the proof.
}

\bibliographystyle{IEEEtran}
\bibliography{impacts_bib_NEW}

\end{document}

%% file: introduction.tex
\section{Introduction}\label{sec:Intro}
\note{Humans routinely perform physical impacts in everyday activities.} 
Examples range from placing objects down energetically, hammering, and assembling furniture components, to running, jumping, and catching, kicking, or batting an object. In all these cases, we benefit from the large, impulsive forces,  
and accelerate or decelerate our own motion or that of objects we manipulate within a very short time.

With cobots and humanoids being on the rise, robots are no longer only used in industry and for clearly defined, repetitive tasks. Instead, they are safely collaborating with humans side by side and are taking over human tasks as needed. However, a lot of these robotic systems are not built to resist impacts. Especially the gears and torque sensors of cobots are prone to be damaged by the impulsive forces. While slight impacts have still been demonstrated with common cobots \cite{Aouaj2020, Yan2024, Dehio2021}, the range of feasible contact velocities is strongly restricted. 

Besides challenges for the hardware, dealing with impacts also is not straightforward for planning and control of the robotic system. 
At an impact, velocity changes have been observed to build up in a matter of few milliseconds \cite{Wang2022, Aouaj2020}. Thus, they cannot be shaped desirably during their occurrence by state-of-the-art control architectures operating at frequencies of few $ \SI{}{\kilo\hertz} $. Hence, a planning and control framework for impact applications needs to be able to deal with sudden velocity changes and large external forces whose exact timing and location might be hard to predict.

\note{To address these problems, and to enable robots to also take on impact tasks, the field of impact-aware robotics has gained increasing attention. Instead of treating impacts as a disturbance to be avoided, they are considered as part of the task. Besides creating robust hardware \cite{Shu2024, Liu2008, Ostyn2024}, impact models are derived and validated \cite{Aouaj2020, Arias2024, Wang2022, Wang2021, Dehio2021}, and planning and control methods \cite{Salehian2018, Stouraitis2020, Khurana2024, Yang2021, Yan2024, Steen2024} are designed.} 

\begin{figure}[tb]
	\vspace{1mm}
	\centering
	\includegraphics[width=0.49\textwidth]{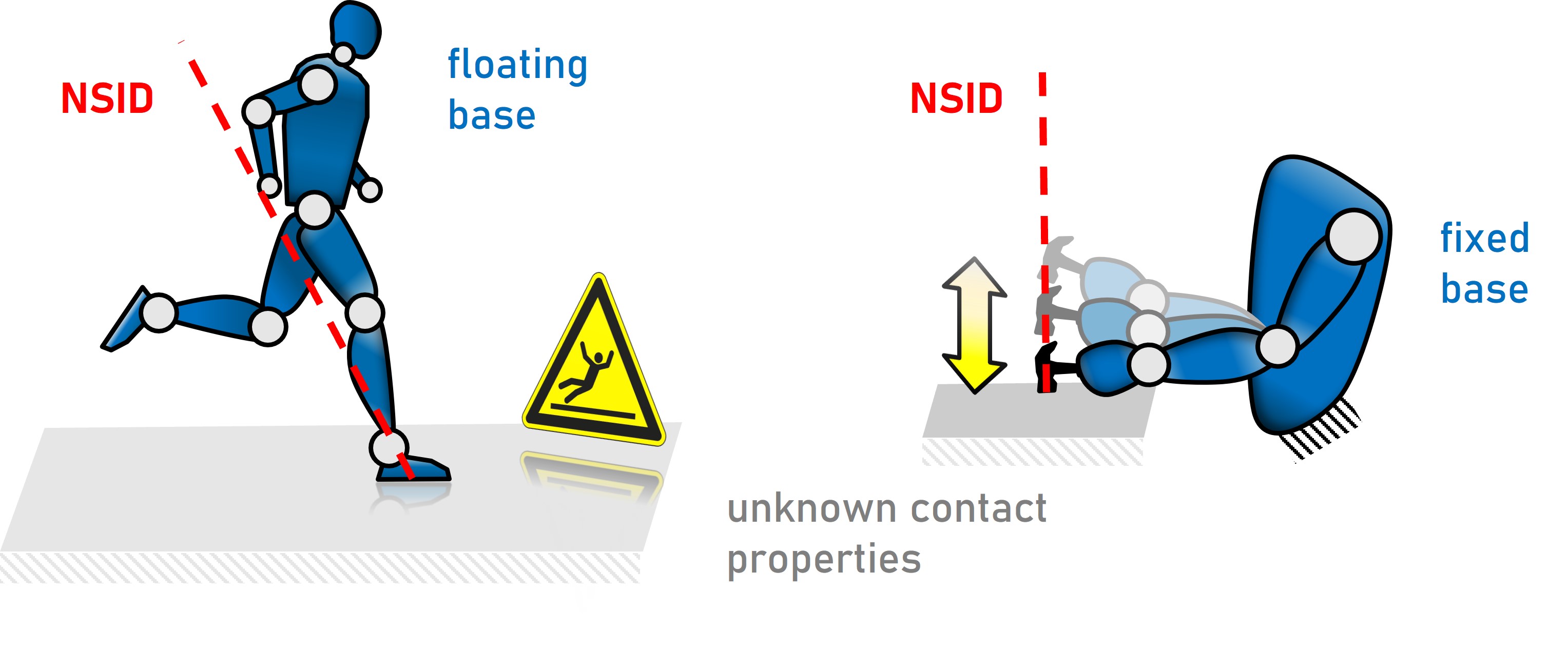}
	\caption{The nonsmooth impact direction is characteristic to impact scenarios of robotic systems and \note{largely independent of contact properties}. It can be \note{beneficial} for various applications, like impacts without slippage and to align pre- and post-impact velocities for repetitive tasks.}
	\label{fig:intro}
\end{figure}

\note{For the development of model-based impact planners and controllers, reliable impact models are required. The impact dynamics of two colliding bodies can be divided into a compression phase and a restitution phase, corresponding to the reduction of the bodies' relative speed to zero and the buildup of a reverse velocity, respectively \cite{Stronge2018}. This comes with material deformations, which are relatively small and short lived for almost rigid bodies. Contact models describing the dynamics of the deformation exist \cite{Gilardi2002} and have been utilized for robotic applications \cite{Yan2024}. However, as the impact dynamics of robots with rigid links colliding with rigid surfaces have been observed to be significantly faster than the (controlled) robot dynamics, impacts are often treated as instantaneous events in robotics \cite{Walker1994, Wang2022, Arias2024}. Under this idealization, forces acting at an impact become impulsive and system velocities nonsmooth. Then, two velocities are related to the time instant of the impact: a pre-impact and a post-impact one. The common model used to map pre- to post-impact generalized velocities of a robotic system considers the robot as a jointed rigid-link system without actuation at the impact. It is founded in nonsmooth mechanics \cite{Brogliato1999} and has been applied in robotics for more then thirty years \cite{Walker1994, Lee2000} for fixed-base manipulation \cite{Proper2023a} as well as locomotion \cite{Westervelt2007,Yang2021}. Recently, it has been experimentally validated for torque-controlled robots \cite{Aouaj2020, Arias2024}. Moreover, extensions and modifications exist, e.g., for kinematic-controlled, non-backdrivable robots \cite{Wang2022, Wang2021}. 
}

\note{While impacts clearly modify system velocities of robots, not all directions of the configuration space or task space are equally affected \cite{Brogliato1999}. Recent works demonstrate that knowledge about impact-invariant and nonsmooth directions of an robotic impact can benefit planning and control \cite{Yang2021, Kirner2024}.}
\note{However, to the best of our knowledge, there is no comprehensive study of such directions in robotic impacts, which analyzes their properties and makes them accessible for further applications.}

\subsection{Contribution}\label{sec:contribution}
\note{In this work, we present a systematic, projection-based analysis of the commonly used, instantaneous, frictionless robotic impact model, which incorporates contact elasticity via Newton’s restitution law \cite{Walker1994}.
Inspired by \cite{Yang2021, Kirner2024}, we separate first the configuration space, and then the task space in an impact invariant subspace as well as a nonsmooth impact direction (NSID). This decomposition is shown to not depend on the restitution coefficient and provides geometrically intuitive insights into effects of the approach direction on the rebound direction.} 

\note{Our work focuses on the NSID, which we identify as a characteristic direction of a robotic impact scenario. Foremost, it is the unique approach direction, which is reversed upon an arbitrarily elastic impact. Only the scaling of the post-impact velocity along the reversed NSID depends on the impact elasticity.} 

\note{Moreover, we discuss the NSID's relation to the inertia ellipsoid showing that the NSID aligns with a principal axis of the ellipsoid if and only if it is perpendicular to the contact normal.}

\note{For the special case of fully inelastic impacts, we show that the NSID can be interpreted as the "non-slippage impact direction" derived in \cite{Kirner2024}. We extend the analysis of \cite{Kirner2024} by explicitly taking contact friction into account. Notably, we prove that an approach  along the NSID generates an impulsive force that is normal to the constraint surface, independent of the frictional properties. Consequently, such an approach prevents post-impact slippage for arbitrarily frictional, fully inelastic impacts.}

\note{The analysis of this paper is directly applicable to backdrivable, rigid-link robotic systems with fixed and floating bases. We address particularities of robots with flexible joints and constrained systems separately.}

\note{Experiments with a passive multi-body system and a torque-controlled robot under various control approaches validate our core statements and their practical relevance.}

\note{Our results shall inform future impact-aware planning or control algorithms and applications ranging from humanoid running to hammering. }

\subsection{Structure of the Work}
\note{We discuss related works in Sec.~\ref{sec:rel_works} and provide the necessary mathematical fundamentals as well as the utilized state-of-the-art impact model in Sec.~\ref{sec:fundamentals}. In Sec.~\ref{sec:model}, the considered impact model is systematically analyzed and the NSID is introduced in configuration and task space. In Sec.~\ref{sec:NSID discussion}, we discuss different interpretations of the NSID in detail. Section~\ref{sec:friction} addresses the special case of fully inelastic impacts. In the following, we extend our analysis to flexible joint robots in Sec.~\ref{sec:el joints}, and to constrained robots in Sec.~\ref{sec:constraint}. Experiments with a passive system and a torque-controlled manipulator are presented in Sections~\ref{sec:experiments} and~\ref{sec:franka_exp}, respectively.
Finally, we conclude the work in Sec.~\ref{sec: discussion}.}

%% file: background.tex
\section{\note{Related Works}}\label{sec:rel_works}

\subsection{\note{Nonsmooth Mechanics and Impact Mechanics}}\label{sec:rel_works:nonsmooth}
\note{Instantaneous impacts in multibody systems have been extensively studied in the mechanics literature, e.g., \cite{Pfeiffer1996, Stronge2018}. Within this context, the field of nonsmooth mechanics is particularly relevant, where seminal contributions include \cite{Moreau1988a, Brogliato1999, Glocker2001}. Although our analysis is not based on the same set-valued or measure-theoretic formalism traditionally used in nonsmooth mechanics, the underlying physical principles 
are identical. Consequently, the direction that we refer to as the NSID naturally appears in these works. In particular, it is expressed as the normal direction to the constraint surface under the kinetic metric in \cite{Brogliato1999}. To the best of our knowledge, however, a dedicated and comprehensive analysis of this direction with robotics applications in mind has not yet been presented. Moreover, our projection-based derivation provides a geometric interpretation that, we believe, makes the concept more accessible to the robotics community.}

\subsection{\note{Impact-Aware Robotics}}\label{sec:rel:impact-aware}
\note{Recently, impact applications have been successfully demonstrated with robotic systems such as pushing \cite{Khurana2024}, fast grasping \cite{Steen2024, Wang2023}, catching \cite{Yan2024}, or jumping \cite{Yang2021}. For that purpose, control frameworks like reference-spreading \cite{Steen2024,Saccon2014} or impact-invariant control \cite{Yang2021} have been proposed, which are capable of dealing with discontinuous control errors appearing at uncertain timing. To exploit the impact dynamics for tasks like pushing or hammering, a dedicated planning of the robot's pre-impact state has been discussed. For example, \cite{Kirner2024} achieves impacts without post-impact slippage by approaching the contact along a configuration-dependent direction, while \cite{Khurana2025} selects the impacting link based on the desired reflected inertia.}
\note{The impulsive force at impact depends on the reflected inertia in the direction of the contact normal \cite{Walker1994}, which can be characterized by the generalized impact ellipsoid, dual to the generalized inertia ellipsoid \cite{Asada1983}. While minimizing reflected inertia is often pursued for safety in human–robot interaction \cite{Gertz1991, Mansfeld2017}, intentional impacts may instead benefit from a large or optimized reflected inertia in the contact direction \cite{Khurana2024,Khurana2025}.}

\note{The work at hand aims to further inform model-based impact-aware planning and control frameworks by providing a systematic analysis of impact-induced discontinuities. The works \cite{Yang2021} and \cite{Kirner2024} are closely related and will be situated with respect to our contribution in the following.}

\subsubsection{\note{Impact Invariant Control}}
\note{The recent work \cite{Yang2021} identifies \emph{impact-invariant} velocity directions in task and configuration space for inelastic impacts based on the nonsmooth impact model \cite{Walker1994}. Their impact-aware tracking controller projects the velocity error onto this subspace in a short time window around the expected impact. This strategy effectively shields the controller from inaccuracies in the estimated impact time, which is demonstrated in bipedal jumping experiments. Although the NSID is not derived in \cite{Yang2021}, the controller rejects all components of the velocity error parallel to the NSID around the impact event.}
\note{While \cite{Yang2021} focuses on the control strategy and its experimental validation, we provide an extensive analysis of the underlying impact dynamics and the associated invariant and nonsmooth subspaces. As opposed to \cite{Yang2021}, we consider the general case of arbitrarily elastic impacts and provide extensions to robots with flexible joints and constrained systems.}

\subsubsection{\note{Targeted Nonslippage Impacts}}
\note{The work \cite{Kirner2024} shows that post-impact slippage can be prevented by performing the impact along a unique, target-dependent approach direction, termed the \textit{"non-slippage impact direction"}. Experiments with a torque-controlled manipulator confirm the theoretic results. The analysis in \cite{Kirner2024} is based on the impact model of \cite{Walker1994} but restricted to systems with non-elastic joints and fully inelastic, frictionless contacts.}

\note{After deriving and interpreting the NSID for impacts with arbitrary coefficient of restitution in the following, we will show that it can be interpreted as the \textit{"non-slippage impact direction"} from \cite{Kirner2024} for the special case of fully inelastic impacts. The shared acronym was deliberately chosen to reflect this connection. Consequently, for inelastic impacts, all results obtained in the more general analysis presented in this work can be leveraged to realize impacts without slippage, for example in flexible-joint robots or constrained systems. Moreover, we extend the results of \cite{Kirner2024} by incorporating contact friction into both the theoretical analysis and the experimental validation for inelastic contacts.}

\section{Fundamentals}\label{sec:fundamentals}

\subsection{Notation and Mathematical Background}\label{sec:fundamentals:math} 
Consider an arbitrary matrix $ {\T{U}\in\Real^{b\times a}} $ with $ {a \geq b > 0} $ of rank $ b $. \note{In the following we will denote the right generalized inverse
\begin{equation}\label{eq:pseudoInv}
	\pseudoInv{U}{\T W} = \inv{\T{W}}\T{U}^T\inv{\left(\T U \inv{\T{W}}\T{U}^T\right)}
\end{equation}
of $ \T U $ as $ \pseudoInv{U}{\T W} $. Therein $ \T{W} \in \Real^{a\times a} $ is a positive definite and symmetric weighting matrix. It holds $ \T U \pseudoInv{U}{\T W} = \T I$, with identity matrix $ \T I $.  In this work, we will mainly encounter generalized inverses when identifying projectors, e.g. of the form
\begin{equation}\label{eq:projector}
	\T P = \pseudoInv{U}{\T W}\T{U}.
\end{equation}
Our analysis strongly relies on properties of projectors which we summarize in \Cref{app:projectors}.}

\subsection{Dynamics and Impact Model}
Consider a robotic system with dynamics modeled by
\begin{equation}\label{eq:RD rigid}
	\M(\q)\ddq + \T{h}(\q,\dq) = \T{\tau} + \T\tau_{\textup{contact}}.
\end{equation}
The $ n $ generalized coordinates are stacked in $ \q\in\Real^n $. We denote the inertia matrix\footnote{Note that, for readability, we will omit arguments of functions after the first introduction, unless they are required for clarity.} as $ \M\in\Real^{n\times n} $.  Coriolis-,  centrifugal, and gravitational torques are contained in $ \T{h} \in \Real^n $. 
\note{The generalized torques are contained in $ \T{\tau} \in \Real^n $. For fully actuated systems, $\T{\tau}$ contains the actuator torques, while in the underactuated case it can be written as $\T{\tau} = \T{B}\T{u}$, with input matrix $\T{B} \in \Real^{n \times u}$ and control input $\T{u} \in \Real^{u}$, $u < n$. The following analysis is independent of the degree of underactuation and we will not have to distinguish between the two cases. Finally, $ \T{\tau}_{\textup{contact}} \in \Real^n $ denotes generalized torques introduced due to interaction with the environment.}

\subsubsection{Task Space Dynamics}
We consider a Cartesian task coordinate $ \x\in\Real^m $ with $ m \leq n $ obtained from the configuration via the forward kinematics $ \x = \T{f}(\q) $. Via the task Jacobian $ {\T{J}(\q) = \frac{\partial \T f(\q)}{\partial \q} \in {\Real^{m\times n}}} $, the task velocities can be obtained as $ \dx = \J\dq $. Differentiating this expression and substituting \eqref{eq:RD rigid}, the task space dynamics model
\begin{equation}\label{eq:task dynamics}
	\M_x(\q) \ddx + \T{h}_x(\q,\dq) = \note{\M_x\J\inv{\M}(\T{\tau} + \T{\tau}_{x,\textup{contact}})}
\end{equation}
is derived. \note{Therein, $ {\M_x = \inv{(\J\inv{\M}\J^T)}  \in \Real^{m\times m}}$ denotes the task space inertia matrix, and the Coriolis, centrifugal and gravitational wrench is expressed as $ {\T{h}_x(\q,\dq) = \M_x\left(\J\inv{\M}\T{h}\left(\q,\dq\right)-\dot{\J}\dq\right)} \in \Real^m$.} 

\subsubsection{Inequality Constraint}
In this work, we consider impacts in the task space of the robot. Thus, the task coordinate is assumed to be subject to a onedimensional inequality constraint $ \varphi(\x) \geq \T 0 $ \note{(cf. Fig.~\ref{fig:A_phi} for an illustrative example)}. Free motion corresponds to an inactive constraint, i.e. $ \varphi(\x) > 0 $, with $ \T\tau_{\textup{contact}} = \T 0 $. For $ \varphi(\x) = 0 $, the robotic system is in contact with the environment. Wrenches $ \T{F}_{x,\textup{contact}} = \Ax^T\lambda $ on the task coordinate result from the contact force $ \lambda > 0 $. Therein, $ \Ax(\x) = \frac{\partial \varphi(\x)}{\partial \x } \in \Real^{1\times m}$ denotes the constraint Jacobian in task space. 

In configuration space, the constraint can be expressed as
\begin{equation}\label{eq:constraint}
	\phi(\q) \coloneq (\varphi \circ \T f)(\q) \ \geq 0.
\end{equation}
The constraint Jacobian $ \A(\q)\in\Real^{1\times n} $ is given by $ {\A(\q) = \frac{\partial\phi(\q)}{\partial\q}} $.
Note that it holds $ \A(\q) = \Ax(\T{f}(\q))\J(\q) $. For an active constraint the generalized contact torque $ {\T\tau_{\textup{contact}} = \A^T\lambda} $ is obtained.

\begin{figure}[tb]
	\centering
	\includegraphics[width=0.49\textwidth]{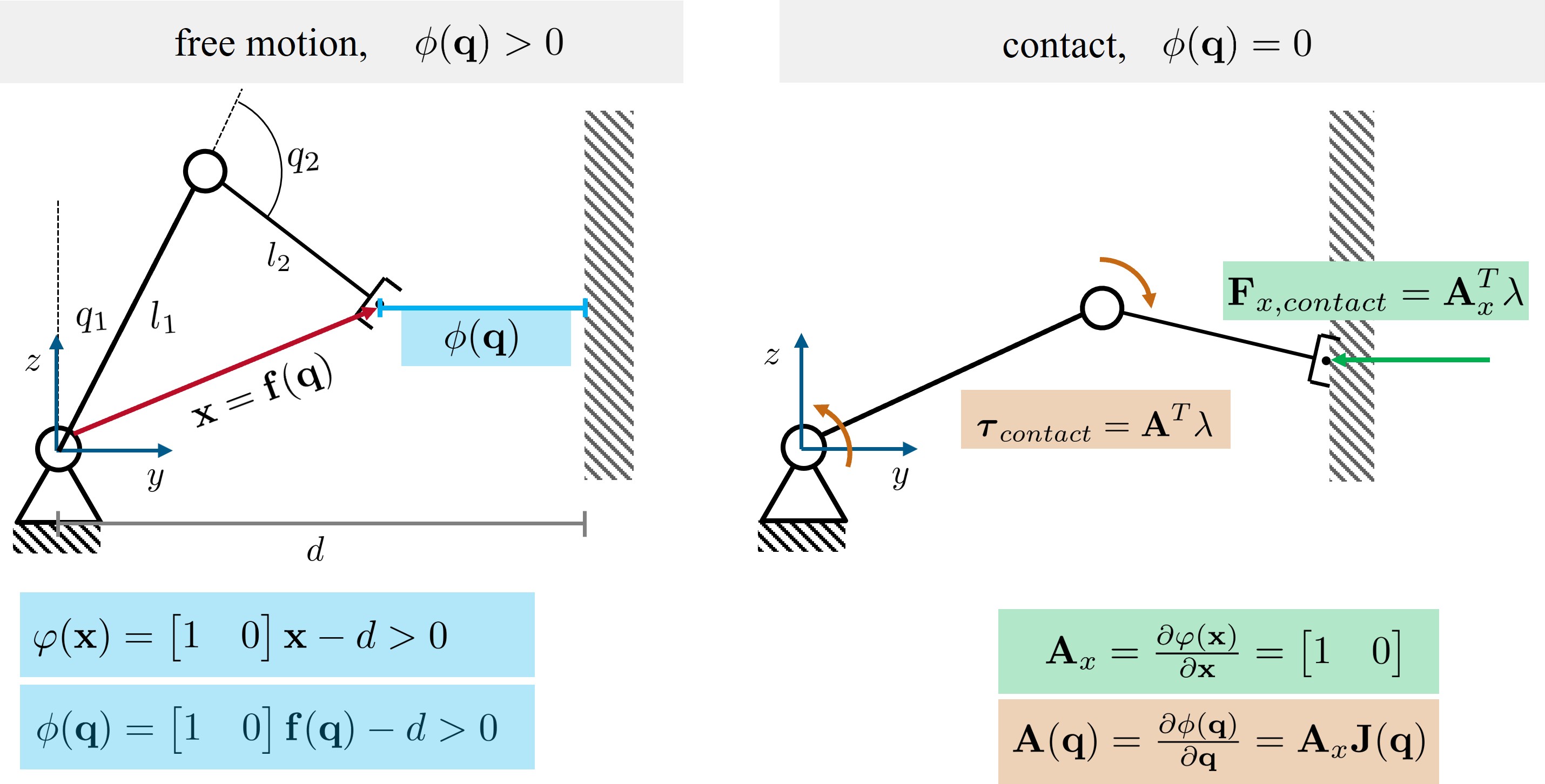}
	\caption{\note{Sketch of a simple robot, whose endeffector is constrained in $ y $-direction. In free motion, $ \phi(\q) > 0$ expresses the minimum distance to the surface. A contact force $ \lambda $ is feasible for $ \phi = 0 $. It becomes impulsive under an impact.}}
	\label{fig:A_phi}
\end{figure}

In the following, we study impacts satisfying the following assumption:
\begin{assumption}\label{a:non singular basic}
	At the impact, the robotic system is in a configuration that is not singular in direction of the inequality constraint. Thus, it holds $ \Ax \neq \T 0 $ and $ \A \neq \T 0 $.
\end{assumption}

For a frictional contact, friction forces can be yielded tangential to the surface $ \varphi(\x) = 0 $. However, we will further use the following assumption for elastic contacts:
\begin{assumption}\label{a:frictionless}
	The contact is frictionless.
\end{assumption}

\subsubsection{Instantaneous Contact Transitions}
At an impact, the robotic system transits from a free-motion to a contact phase at a non-vanishing contact velocity. 
\begin{assumption}\label{a:instantaneous}
	The impact is instantaneous. 
\end{assumption}
This is motivated by the fact that the contact dynamics are often significantly faster than the (link-side) robot dynamics. When they are either in or below the order of magnitude of the sampling time, 
effects of the impact cannot be mitigated by a controller \textit{during} the impact. Instead, either precautions must be taken to achieve desired post-impact behavior or possibly undesired post-impact behavior needs to be corrected afterwards.

Assuming that the configuration is constant over an impact, it becomes immediately clear that the inequality constraint \eqref{eq:constraint} can only remain satisfied if the system velocities become discontinuous at the moment of the impact. The contact force $ \lambda $ then becomes infinitely large.
Integrating the robot dynamics \eqref{eq:RD rigid} over the infinitesimal impact duration $ \Delta t \rightarrow 0 $, one obtains the impact equation \cite{Walker1994, Brogliato1999} 
\begin{equation}\label{eq:impact_eq rigid}
	\M(\q)(\dqplus - \dqminus) = \A^T(\q)\Lambda 
\end{equation}
relating system velocity jumps introduced by the impact to an impulsive force $ \Lambda = \lim_{\Delta t \rightarrow 0}\int_{0}^{\Delta t} \lambda dt $ with $ \Lambda \in \Real $.
Therein, $ \dqminus $ denotes the pre-impact generalized velocity, whereas $ \dqplus $ refers to the post-impact one. We will utilize the superscripts $ - $ and $ + $ to distinguish pre- and post-impact values of system velocities, respectively, throughout the paper.

\subsubsection{Impact Elasticity}
Depending on the properties of the contact, the kinetic energy of the robotic system 
can either be dissipated or, at least partially, be preserved. Several contact models exist describing this relation. 
In robotics, Newton's restitution law \cite{Gilardi2002} has commonly been used \cite{Walker1994, Lee2000, Khurana2024}\note{, which often allows closed-form solutions.}
\begin{assumption}[\textbf{Newton's restitution law}]\label{a:restitution}
	For the post-impact contact velocity it holds
	\begin{equation}\label{eq:el assumption}
		\dot\phi^+= \A(\q)\dqplus = -e\A(\q)\dqminus = -e \dot\phi^-,
	\end{equation}
	with restitution factor $ e \in [0,1] $.
\end{assumption}
 For a fully inelastic impact, i.e. $ e = 0 $, post-impact velocities $ \dot\phi^+ $ in constraint direction vanish. Conversely, the pre-impact value $ \dot\phi^- $ is reversed for a fully elastic impact, i.e. $ e = 1 $. Note that, for the particular case of frictionless impacts, which we consider here, the squared restitution factor directly relates pre- to post-impact kinetic \note{energy.} In particular, for $ e = 1 $, the kinetic energy of the robot is conserved over the impact.

Combining \eqref{eq:el assumption} and \eqref{eq:impact_eq rigid}, the impulsive force is obtained as
\begin{equation}\label{eq:Lambda}
	\Lambda = -(1+e)\inv{\left(\A\inv{\M}\A^T\right)}\A\dqminus.
\end{equation} 
Finally, the map
	\begin{subequations}\label{eq: el impact model rigid}
		\begin{align}
		\dqplus &= \left(\T I - \left(1+e\right)\underbrace{\Jmp{A}\A}_{\pseudoInv{A}{\T M}\A}\right) \dqminus \\
		\label{eq: Q}
		&= \underbrace{\left(\T I - \left(1+e\right)\T P_q\right)}_{\coloneq\elMap(\q)} \dqminus
	\end{align}
	\end{subequations}
between pre- and post-impact velocities follows from \eqref{eq:impact_eq rigid} and \eqref{eq:Lambda} without any further assumptions. Therein, we identify $ \T{P}_q = \pseudoInv{A}{\T M}\A $ to be a projector on $ \range({\inv{\M}\A^T}) $ and denote the matrix mapping $ \dqminus $ to $ \dqplus $ as $ {\T{Q}(\q)\in\Real^{n\times n}} $.

The impact map \eqref{eq: el impact model rigid} has been applied in the literature to a broad range of robotic systems.
Commonly, a fully inelastic case, i.e. $ e=0 $, is assumed \cite{Aouaj2020, Yang2021, Kirner2024}.

\subsubsection{Remarks on the applicability to generalized robotic systems}\label{sec:model_for_generalized}

Common \textit{torque-controlled robots} can be considered as flexible joint robots with comparably stiff joints. In \cite{Aouaj2020, Arias2024} it has been shown that the map \eqref{eq: el impact model rigid} can still predict post-impact behavior of this class of robotic systems for the considered case of $ e = 0 $. In \cite{Arias2024} it was further shown that the predictions can be improved, if the apparent motor inertia rendered by the inner torque control loop is added upon the link-side inertia matrix.

\textit{Kinematically-controlled robots: }The map \eqref{eq: el impact model rigid} treats the robot as a jointed multibody system at the instant of the impact. While the previously mentioned works show that it is accurate in experiments with backdrivable robotic systems, the works \cite{Wang2021, Wang2022} suggest that the model is less suitable to predict post-impact motions of kinematically-controlled systems.

\textit{Fixed and floating base systems:} It is important to note that the model derives for an arbitrary set of generalized coordinates without any assumptions on a fixed or floating base. Also, the degree of underactuation does not affect the derivation. Thus, the model can be directly applied to robotic systems with either type of base.

%% file: model2.tex
\section{Non-smooth impact directions in generalized velocity space and task velocity space}\label{sec:model}
In this section, we perform a detailed analysis of how the system velocities are affected by the impact modeled by \eqref{eq: el impact model rigid} and introduce the nonsmooth impact direction (NSID). The analysis is performed in configuration space in Sec.~\ref{sec:model JointSpace} and in task space in Sec.~\ref{sec:model TaskSpace}.

\subsection{Analysis of the Impact Model in Configuration Space}\label{sec:model JointSpace}
Consider the matrix $ \elMap $ from \eqref{eq: el impact model rigid}. It is evident that for $ e = 0 $ it becomes a projector on $ \ker(\A) $.
In the following, we analyze the eigenspaces for the general case of $ e \in [0,1] $ dividing the admissible configuration space in an impact invariant space \cite{Yang2021} and a nonsmooth impact direction (NSID), which generalizes the non-slippage impact direction introduced in~\cite{Kirner2024}.

\subsubsection{Impact Invariant Velocities}
\begin{theorem}[\textbf{impact invariant generalized velocities}]\label{thm:dqiv}
	For a pre-impact velocity $ \dqminus = \dq\iv $ with an arbitrary $ {\dq\iv \in \ker(\A)} $, the post-impact velocity is obtained as $ \dqplus = \dqminus $.
\end{theorem}
\begin{proof}
	The velocity $ \dq\iv $ is in the null space of $ \Pma $. Thus, it holds $ \dqplus = (\T I - (1+e)\Pma)\dqminus = \dqminus $.
\end{proof}
Theorem~\ref{thm:dqiv} can be reformulated as follows: 
\begin{itemize}
	\item The matrix $ \elMap $ has the $ (n-1) $-fold eigenvalue of $ 1 $ where the corresponding eigenspace is $ \ker(\A) $.
	\item All $ {\dq\iv \in \ker(\A)} $ are \textit{impact invariant} forming the impact invariant subspace.
\end{itemize}
Following \cite{Yang2021}, we refer to velocities that are not affected by the impact as \note{\textit{impact invariant}.} 
Since the invariant velocities $ \dq\iv $ do not involve a velocity in constraint direction (i.e. $ {\dot{\phi} = \A\dq\iv = 0} $) an additional contribution in constraint direction is required for an impact.

\subsubsection{Nonsmooth Impact Direction}
\begin{theorem}[\textbf{NSID in configuration space}]\label{thm:dqimp}
	For a pre-impact velocity $ \dqminus = \dq\imp(\nu,\q) $, in the form
	\begin{equation}\label{eq:dqimp}
		\dq\imp(\nu,\q) := \nu\inv{\M}(\q)\A^T(\q),
	\end{equation}
	with an arbitrary scalar $ \nu < 0 $, it holds $ \dqplus = - e \dqminus $. 
\end{theorem}
\begin{proof}
	The velocity $ \dq\imp(\nu,\q) $ is in the range of $ \Pma(\q) $ for an arbitrary $ \nu\in\Real $. Thus, for $ \dqminus = \dq\imp $ it holds $ {\Pma\dqminus = \dqminus} $ and \eqref{eq: el impact model rigid} evaluates as $ {\dqplus = - e \dqminus} $. The restriction of the scalar $ \nu $ to $ \nu < 0 $ ensures an approach of the inequality constraint from the admissible side, as it then holds $ {\dot\phi^-(\q) = \nu\A\inv{\M}\A^T  < 0 } $.
\end{proof}
Theorem~\ref{thm:dqimp} directly implies that $ \elMap $ has the onefold eigenvalue of $ -e $, where the corresponding eigenspace is spanned by $ \dq\imp $. Velocities from this onedimensional space are \textit{nonsmooth} at the \textit{impact}. It is important to note that only \textit{one direction} of the subspace spanned by $ \dq\imp $ is admissible, which is reflected in the condition $ \nu < 0 $. Thus, we will refer to $ \dq\imp $ as the \textit{nonsmooth impact direction} (NSID) \note{in the following.}

\subsubsection{General Impact Velocities}\label{sec:general dq}
Based on the results from the previous subsections, we summarize the $ n $ eigenvalues and corresponding eigenspaces of $ \elMap $ in \Cref{tab:eig Maps}(A). 
\begin{table}
	\vspace{2mm}
	\begin{center}\small
		\setlength\extrarowheight{3pt}
		\centering
		(A) configuration space\\
		\vspace{1mm}
		\begin{tabular}{ c | c c }
			\hline
			eigenvalue & dimension of eigenspace & eigenspace \\ 
			\hline
			$ -e $ & $ 1 $ & $ \range(\inv{\M}\A^T) $ \\
			$ 1 $ & $ n-1 $ & $ \ker(\A) $ \\  
			\hline 
		\end{tabular}
		
	\end{center}%
	\vspace{1mm}
	
	\begin{center}\small
		\setlength\extrarowheight{3pt}
		\centering
		(B) task space\\
		\vspace{1mm}
		\begin{tabular}{ c | c c }
			\hline
			eigenvalue & dimension of eigenspace & eigenspace \\ 
			\hline
			$ -e $ & $ 1 $ & $ \range(\inv{\M_x}\Ax^T) $ \\
			$ 1 $ & $ m-1 $ & $ \ker(\Ax) $ \\  
			\hline  
		\end{tabular}
		
	\end{center}%
	\vspace{1mm}
	\caption{Eigenspaces of $ \elMap $ (a) and $ \taskMap $ (b)}
	\label{tab:eig Maps}
\end{table}
Given these, we can decompose any admissible pre-impact velocity $ \dqminus $ as
\begin{equation}\label{eq:dqminus zerlegt}
	\dqminus = \underbrace{\Pma \dqminus}_{\dq\imp} + \underbrace{(\T I - \Pma)\dqminus}_{\dq\iv},
\end{equation}
where the first contribution $ \dq\imp(\nu,\q) $ aligns with the NSID. The specific scaling $ \nu<0 $ is uniquely defined by $ \dqminus $. The second contribution $ \dq\iv $ is impact invariant. An impact under \eqref{eq:dqminus zerlegt} yields the impulsive force from \eqref{eq:Lambda} as $ {\Lambda = -(1+e)\nu} $. For the post-impact velocities it holds \begin{equation}\label{eq:dqplus}
	\dqplus = -e\dq\imp +\dq\iv.
\end{equation}

\subsubsection{Inverse Impact Model}
The presented analysis straightforwardly allows to invert the impact model \eqref{eq: el impact model rigid}.
\begin{corollary}
	Given a \note{feasible}, desired post-impact velocity $ \dqplus $ in generalized velocity space and a configuration $ \q $, the pre-impact velocities realizing the desired post-impact behavior can be computed as 
\note{	\begin{equation}\label{eq: inverse model dq}
		\dqminus = \begin{cases}
			 \left(\T I - \left(1+\frac{1}{e}\right)\Pma \right)\dqplus & \text{for } 0 < e \leq 1\\
			\dqplus + \nu\inv{\M}\A^T & \text{otherwise},
		\end{cases}
	\end{equation}
with arbitrary $ \nu < 0 $.
}
\end{corollary}
\begin{proof}
	\note{First consider $ 0 < e \leq 1 $. Plugging \eqref{eq: inverse model dq} in \eqref{eq: el impact model rigid} and utilizing idempotence of the projector $ \T P_q $, we obtain $ {\dqplus = \T I \dqplus} $. Moreover, plugging \eqref{eq: el impact model rigid} in \eqref{eq: inverse model dq}, yields $ {\dqminus =  \T I \dqminus} $.}
	\note{The same can be shown for $ e = 0 $ considering that feasible post-impact velocities satisfy $ \dqplus = (\T I - \Pma)\dqplus $.}
\end{proof}
\note{Note that, for $ e = 0 $, feasible $ \dqplus $ can thus be rendered by infinitely many $ \dqminus $, each of them aligning the NSID.}

We summarize the results of this analysis in the configuration space in Fig.~\ref{fig:Q}.

\begin{figure}
	\centering
	\includegraphics[width = 0.49\textwidth]{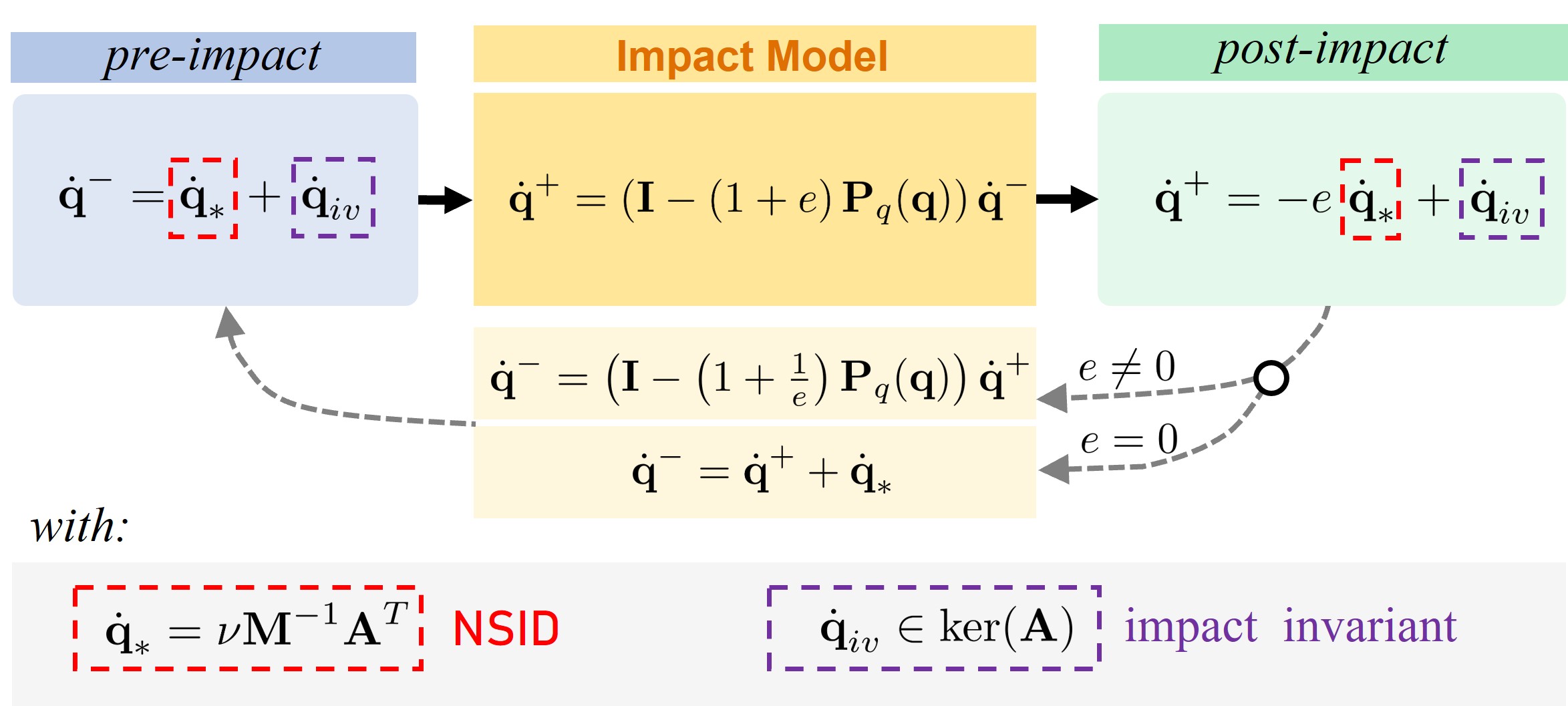}
	\caption{\note{Relation between pre- and post-impact velocities in configuration space. A contribution to $ \dqminus $ aligning with the NSID is scaled by the impact, whereas impact invariant contributions are preserved.}}
	\label{fig:Q}
\end{figure}

\subsection{A Task Space Impact Model}\label{sec:model TaskSpace}
\begin{theorem}[\textbf{task space impact map}]\label{thm:task map} 
	The unique relation
	\begin{align}\label{eq:task impact model}
		\dxplus = \underbrace{\left(\T I - (1+e)\Px\right)}_{\coloneq\taskMap(\q)}\dxminus
	\end{align}
	holds between the post-impact task velocities and the pre-impact ones, with the projector
	\begin{align}\label{eq:Px}
		\Px &= \pseudoInv{A}{\T M_x}_x\Ax
	\end{align} 
	and the mapping $ \taskMap(\q) \in \Real^{m\times m} $.
\end{theorem}
\begin{proof}
	The map \eqref{eq:task impact model} can be straightforwardly derived by premultiplying \eqref{eq: el impact model rigid} with the task Jacobian and substituting $ \A = \Ax\J$.	
\end{proof}
Moreover, the impulsive force $ \Lambda $ from \eqref{eq:Lambda} can be uniquely obtained from the task velocities $ \dxminus $ as
\begin{equation}\label{eq:Lambda task}
	\Lambda = -(1+e)\inv{\left(\Ax\inv{\M_x}\Ax^T\right)}\Ax\dxminus.
\end{equation}
\Cref{thm:task map} provides a direct mapping between $ \dxminus $ and $ \dxplus $, which is valid also for redundant robots, does not require knowledge about null space velocities, and derives from \eqref{eq: el impact model rigid} without any further assumptions. This is due to the fact that the constraint acts on the task coordinate. Indeed, we can show the following.
\begin{lemma}[\textbf{impact invariance of null space velocities}]\label{thm: dqn}
	Any $ \dq $ in the null space of the task, i.e. a $ \dqminus \neq \T 0 $ fulfilling $ {\J\dq = \T 0} $, is impact invariant.
\end{lemma}
\begin{proof}
	\note{We consider a constraint on the task coordinate $ \x $, which implies that $ \ker(\J) \subseteq \ker(\A) $.  $ \ker(\A) $ is equivalent to the eigenspace of $ \elMap $, corresponding to the eigenvalue of $ 1 $. Thus, a $ \dq_n \in \ker(\J)$ is mapped to itself by $ \elMap $, i.e. it is impact invariant.}
\end{proof}

\subsubsection{Impact Invariant and Nonsmooth Task Velocities}
Given the map \eqref{eq:task impact model}, we can separate admissible task velocities in impact invariant task velocities and a task space NSID, equivalently as in the configuration space.
\begin{theorem}[\textbf{impact invariant task velocities}]\label{thm:task invariant}
	Any pre-impact task velocity $ \dxminus = \dx\iv $ with $ {\dx\iv \in \ker(\Ax)} $ is impact-invariant, i.e. it holds $ {\dxplus = \dxminus = \dx\iv} $.
\end{theorem}
\begin{theorem}[\textbf{NSID in task space}]\label{thm:dximp}
	For an approach velocity $ \dxminus = \dx\imp $ in task space, with
	\begin{equation}\label{eq: dximp}
		\dx\imp(\nu,\q) \coloneq \nu \inv{\M_x}(\q)\Ax^T(\T f(\q)) = \nu \J\inv{\M}\A^T
	\end{equation}
	and an arbitrary scalar $ \nu < 0 $, it holds $ \dxplus = -e\dxminus $. 
\end{theorem}
We refer to the task velocity direction represented by \eqref{eq: dximp} as the NSID in task space. 

The proofs of \Cref{thm:task invariant}~and~\ref{thm:dximp} are performed analogously as for the configuration space. From these theorems the eigenspaces of $ \taskMap $ summarized in \Cref{tab:eig Maps}(B) can be directly deduced. We can now express any admissible pre-impact velocity as the following linear combination:
\begin{equation}\label{eq:dxminus decomposed}
	\dxminus = \underbrace{\Px\dxminus}_{\dx\imp} + \underbrace{(\T{I}-\Px)\dxminus}_{\dx\iv},
\end{equation}
with an NSID contribution $ \dx\imp = \Px\dxminus $ and an impact-invariant contribution $ \dx\iv = (\T I - \Px)\dxminus $. For the post-impact velocity obtained for an impact under $ \dxminus $ from \eqref{eq:dxminus decomposed} it then holds
\begin{equation}\label{eq:dxplus}
	\dxplus = -e\dx\imp + \dx\iv.
\end{equation}

It is straightforward that, also for the task space, an inverse impact model can be formulated.
\begin{corollary}
	Given a \note{feasible,} desired post-impact velocity $ \dxplus $ in task space and a configuration $ \q $, the pre-impact velocities realizing the desired post-impact behavior can be computed as 
\note{\begin{equation}\label{eq: inverse model dx}
		\dxminus =
		\begin{cases}
			 \left(\T I - \left(1+\frac{1}{e}\right)\Px \right)\dxplus & \text{for } 0 < e \leq 1 \\
			 \dxplus + \nu\inv{\M}_x\Ax^T & \text{otherwise},
		\end{cases}
	\end{equation}
with arbitrary $ \nu < 0 $.}
\end{corollary}

We summarize the results of this subsection in Fig.~\ref{fig:X}.

\begin{figure}
	\centering
	\includegraphics[width = 0.49\textwidth]{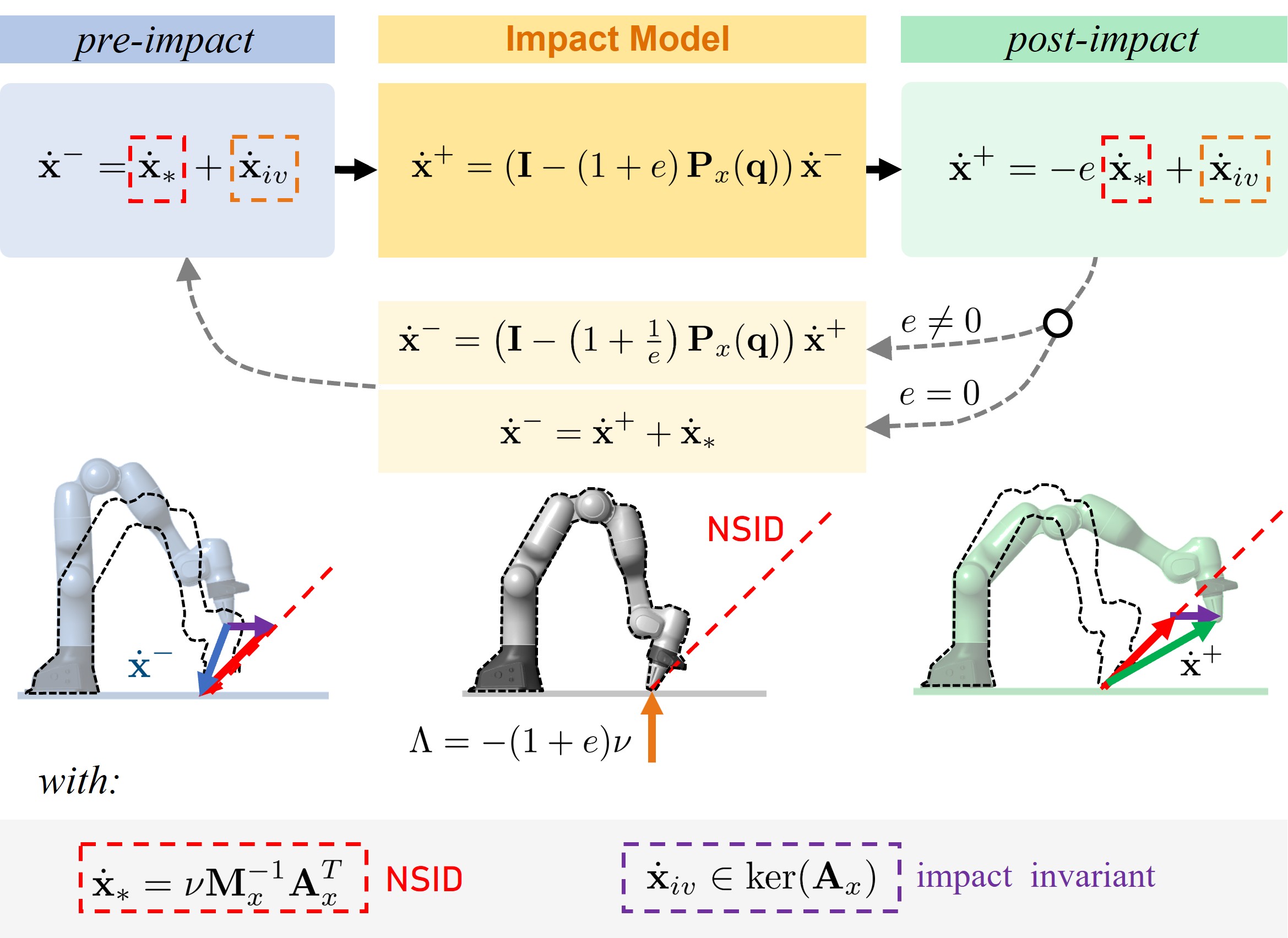}
	\caption{\note{Relations between pre- and post-impact task velocities as established by the map \eqref{eq:task impact model}. Contributions aligning the NSID are scaled by $ -e $, whereas impact invariant contributions are preserved by the impact.}}
	\label{fig:X}
\end{figure}

\textit{Remark:} It is important to note that the task space NSID can contain both translational and angular components. In the schematic visualizations of this work, we will often depict the translational component only, as it can be illustrated nicely. Also, note that, even though Fig.~\ref{fig:X} as well as the following figures, serve as conceptual illustrations, they always depict the actual translational NSID computed for the shown robotic system in the respective configuration. 

%% file: interpretations.tex
\section{Interpretations of the NSID}\label{sec:NSID discussion}
The presented results already suggest the NSID as a characteristic direction describing a robotic impact scenario. The remainder of this work focuses on this particular direction. \note{In this section, we discuss different interpretations.}

\subsection{NSID as the Unique Approach Direction Aligning Pre- and Post-Impact Velocities}\label{sec:model:NSID aligning}
According to Theorems~\ref{thm:dqimp} and \ref{thm:dximp}, the NSID is the unique, admissible approach direction of a configuration on the constraint surface, under which pre- and post-impact velocities align in the respective space. The scaling of the post-impact value with respect to the pre-impact one is provided by the negative restitution factor $ -e $ and thus the contact elasticity.

For a non-redundant robotic system, it is evident that a reversal of the direction of the generalized velocity upon an impact implies the one of the task velocity and vice versa. For a redundant robotic system, only the first implication is fulfilled, in general. An approach under $ \dxminus = \dx\imp $ can then be rendered by infinitely many $ \dqminus $, which can all be obtained as
\begin{align}\label{eq:all dqminus}
	\dqminus &= \pseudoInv{J}{\M}\dx\imp + \T Z^T\T v^- = \dq\imp + \T Z^T\T v^-.
\end{align}
Therein $ \T Z \in \Real^{(m-n)\times n} $ denotes an arbitrary null space base such that it holds $ \J \T Z^T= \T 0 $. The vector $ \T{v}^- \in \Real^{n-m} $ represents an arbitrary pre-impact null space velocity. We utilized the particular choice of the weighting matrix $ \T W  = \M $, to express $ \dqminus $ as a linear combination of the NSID in the configuration space and a null space contribution. Note that the scale $ \nu $ in $ \dq\imp $ is given by $ \dxminus $. For $ \dqplus $ it then straightforwardly holds $ \dqplus =  -e\dq\imp + \T{Z}^T\T{v}^- $, i.e. the contributions in the null space of the task as well as the ones\footnote{Note, that $ \dq\imp $ does contain null space contributions, in general.} included in $ \dq\imp $ are preserved over the impact. 

\note{With respect to the impact elasticity, the two corner cases $e = 1$ and $ e = 0$ are particularly interesting. A fully elastic impact, i.e. $ e = 1 $, under the NSID reverses the approach velocities, as displayed in Fig.~\ref{fig:el concept}. We dedicate the full Sec.~\ref{sec:friction} to the case of $ e = 0 $.} 

\begin{figure}
	\centering
	\includegraphics[width = 0.47\textwidth]{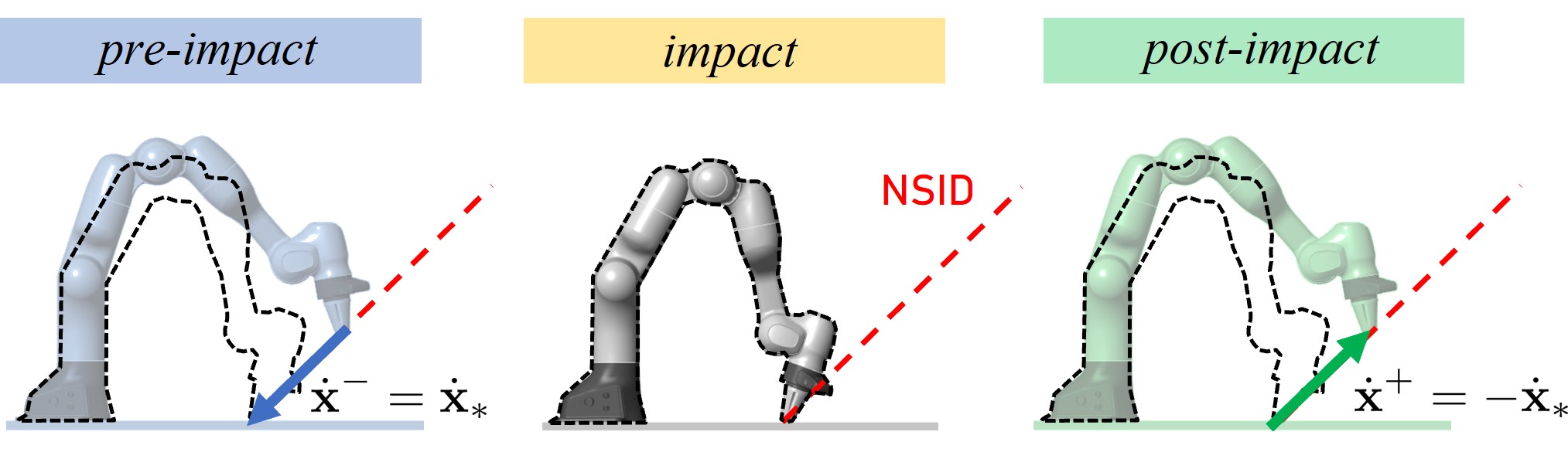}
	\caption{For a fully elastic impact under the NSID, the post-impact velocity is the reversed pre-impact velocity.}
	\label{fig:el concept}
\end{figure}

\subsection{NSID as the Direction of the Velocity Jump}
Consider \eqref{eq:impact_eq rigid} mapping an impulsive force $ \Lambda $ in constraint direction to the impact induced jump $ {\Delta\dq = \dqplus - \dqminus} $ in generalized velocities. It is well known \cite{Brogliato1999} that it holds
\begin{equation}\label{eq:jump}
	\Delta\dq = \inv{\M}(\q)\A^T(\q)\Lambda,
\end{equation}
with $ \Lambda > 0 $. Thus, the following can be concluded.
\begin{corollary}\label{cor:jump}
	The direction of the velocity jump $ \Delta\dq $ in configuration space and $ \Delta\dx = \J\Delta\dq $ in task space aligns with the negative NSID in the respective space.
\end{corollary}

Note that \eqref{eq:jump} is derived solely from integrating the robot dynamics \eqref{eq:RD rigid} over the infinitesimal impact. The interpretation of the NSID as the direction of the velocity jump thus only depends on \Cref{a:instantaneous} of an instantaneous impact. No contact model like the restitution model of \Cref{a:restitution} is needed. Moreover, we can also partly obtain the first interpretation of the NSID as the unique direction aligning pre- and post-impact velocity directions without utilizing a contact model: reformulating \eqref{eq:jump}, we obtain
\begin{align}\label{eq:dqplus from jump}
	\dqplus &= \Delta\dq + \dqminus= \inv{\M}\A^T\Lambda + \dqminus.
\end{align}
With that, it can be concluded without relying on \Cref{a:restitution} that only for a $ \dqminus $ aligning $ \Delta\dq $ and thus the NSID, a $ \dqplus $ aligning $ \dqminus $ is obtained. Only the scaling of $ \dqplus $ w.r.t. $ \dqminus $ after an impact under the NSID depends on the impulsive force and thus the contact.
This emphasizes that the NSID, as the unique nonsmooth direction, is a characteristic direction of a robotic impact scenario, determined by the robot's kinematic and inertial properties, the constraint, and the robot configuration, but not \note{the impact elasticity}. 

\subsection{NSID and the Inertia Ellipsoid}\label{sec:model:ellipsoid}
\begin{theorem}[\textbf{perpendicular NSID}]\label{cor:ellipsoid}
	The task space NSID aligns with the Euclidean normal of the constraint surface, whenever an eigenvector of the task space inertia matrix aligns with the task space constraint Jacobian and thus the surface normal as well.	
\end{theorem}
\begin{proof}
	First, note that the task space constraint Jacobian $ \Ax^T $ is perpendicular to the constraint by definition. Premultiplying \eqref{eq:jump} with the task Jacobian $ \J $, we obtain
	$	\Delta\dx = \inv{\M_x}\Ax^T\Lambda $. Consider the special case of an impact, for which the impulsive wrench $ {\T F = \Ax^T\Lambda \in \Real^m} $ and thus $ \Ax^T $ aligns with the $ i $-th eigenvector of the task space inertia matrix $ \M_x $ corresponding to the eigenvalue $ \alpha_i $. Then, it holds:
	\begin{itemize}
		\item $ \Delta\dx $ aligns $ \Ax^T $,
		\item the NSID \eqref{eq: dximp} takes the form $ \dx\imp = \nu \alpha_i^{-1}\Ax^T  $, i.e. it also aligns $ \Ax^T $, and
		\item the projector $ \Px $ becomes the orthogonal projector $ {\Px = \Ax^T\inv{\left(\Ax\Ax^T\right)}\Ax} $.
	\end{itemize}
	For an impact under the NSID further
	\begin{itemize}
		\item $ \dxplus = - e \nu \alpha_i^{-1} \Ax^T  $ holds, i.e. $ \dxplus $ and $ \dxminus $ align $ \Ax^T $.
	\end{itemize}
\end{proof}

Consequently, an NSID approach perpendicular to the surface always comes either with a maximization or minimization of the transmitted impulsive force for the particular configuration and absolute value of the approach velocity. The relation can nicely be visualized with an inertia ellipsoid\footnote{Note that the ellipsoid corresponds to the generalized impact ellipsoid \cite{Walker1994}, which is equivalent to the inverse generalized inertia ellipsoid \cite{Asada1983}.} defined as $ \left\{\T{F}\in\Real^m : \T{F}^T\inv{\M_x}\T{F}\leq 1 \right\} $. The NSID is perpendicular on the constraint surface iff a principal axis of the ellipsoid is perpendicular to the surface as well.

Notably, these results match an initial assumption in \cite{Khurana2024}. There, it was assumed that the robot's reflected inertia in impact direction is sufficiently high, such that the pre-impact momentum in other directions do not induce a significant deviation of the post-impact motion from the impact direction. We can reformulate that as follows: the impact is assumed to be performed under a configuration with almost perpendicular NSID.

\subsubsection*{Examples} Figure~\ref{fig:ellipsoid} compares pre- with post-impact task velocities for two configurations of the Franka Research 3 robot impacting the horizontal ground with the end-effector. To be able to nicely depict the relations, we treat planar, translational task space motions, without loss of generality.
We consider multiple pre-impact task velocities, depicted as arrows in the figure. They feature the same magnitude but different impact angles. The resulting post-impact task velocities are evaluated exemplary for $ e = 0.3 $. They are scaled and rotated with respect to the pre-impact velocities, as expected. The only exception is the NSID approach under which pre- and post-impact velocities align.

\begin{figure}
	\centering
	\vspace{0.75mm}
	\includegraphics[width = 0.45\textwidth]{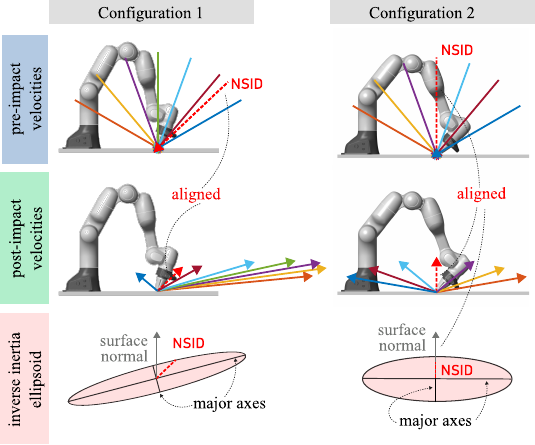}
	\caption{The NSID (here: in task space) is the unique approach direction, for which pre- and post-impact approach velocities are aligned. In the special case when a major axis of the inverse generalized inertia ellipsoid aligns with the normal of the constraint surface, the NSID also aligns with them. It is then perpendicular to the constraint surface.}
	\label{fig:ellipsoid}
\end{figure}

In configuration~1, the surface normal does not align with a principal axis of the inertia ellipsoid. Therefore, the NSID is not perpendicular on the constraint surface. Conversely, configuration~2 has been chosen such that the NSID is perpendicular to the surface (cf. config.~$\perp$ from \cite{Kirner2024} and its computation there). In that case, also a principal axis of the ellipsoid aligns with the NSID, which matches with the presented theoretic results.

As opposed to a robotic system, the translational inertia ellipsoid of one single rigid body is always spherical. Indeed, one can show that the translational NSID of single rigid bodies is always perpendicular to the surface, which matches the intuition. The existence of a rotational component depends on the location of the center of mass (CoM) with respect to the contact point.

%% file: friction_analysis_NEW2.tex
\section{\note{NSID and Fully Inelastic Impacts}}\label{sec:friction}
\note{Several works on impact-aware robotics assume fully inelastic impacts. This implies that the robot remains in contact with the surface after the impact and does not bounce. It holds $ e = 0 $ and thus \eqref{eq:el assumption} reduces to $ \dot\phi^+ = 0 $. However, sliding motions along the contact surface are feasible. Given the previous results, it is straightforward to recover the main result from \cite{Kirner2024}:}

\begin{corollary}\label{cor:non-slippage}
	\note{The NSID in the (configuration / task) space is the unique approach direction of a configuration $ \q $, under which post-impact velocities in the respective space vanish for the case of a fully inelastic, frictionless impact.}
\end{corollary}
\note{Thus, it holds
\begin{itemize}
	\item $ \dqplus = \T 0 $ iff $ \dqminus = \dq\imp $, with arbitrary $ \nu < 0 $, and
	\item $ \dxplus = \T 0 $ iff $ \dxminus = \dx\imp $, with arbitrary $ \nu < 0 $.
\end{itemize}}
\note{Indeed, for an approach under the task space NSID the \textit{targeted non-slippage impact problem} introduced and addressed in \cite{Kirner2024} is solved. There, the goal was to obtain an impact at a target $ \x $ such that it holds $ \dxplus = \T 0$. All solutions for a given $ \q $ solving $ \x = \T{f}(\q) $ have been shown to align with a unique task space approach direction, which corresponds to the NSID in task space. For $ e = 0 $, we can thus interpret the NSID as the "\textit{non-slippage impact direction}"\cite{Kirner2024} corresponding to $ \q $ (cf. Fig.~\ref{fig:non slippage concept}). }

\begin{figure}
	\centering
	\includegraphics[width = 0.47\textwidth]{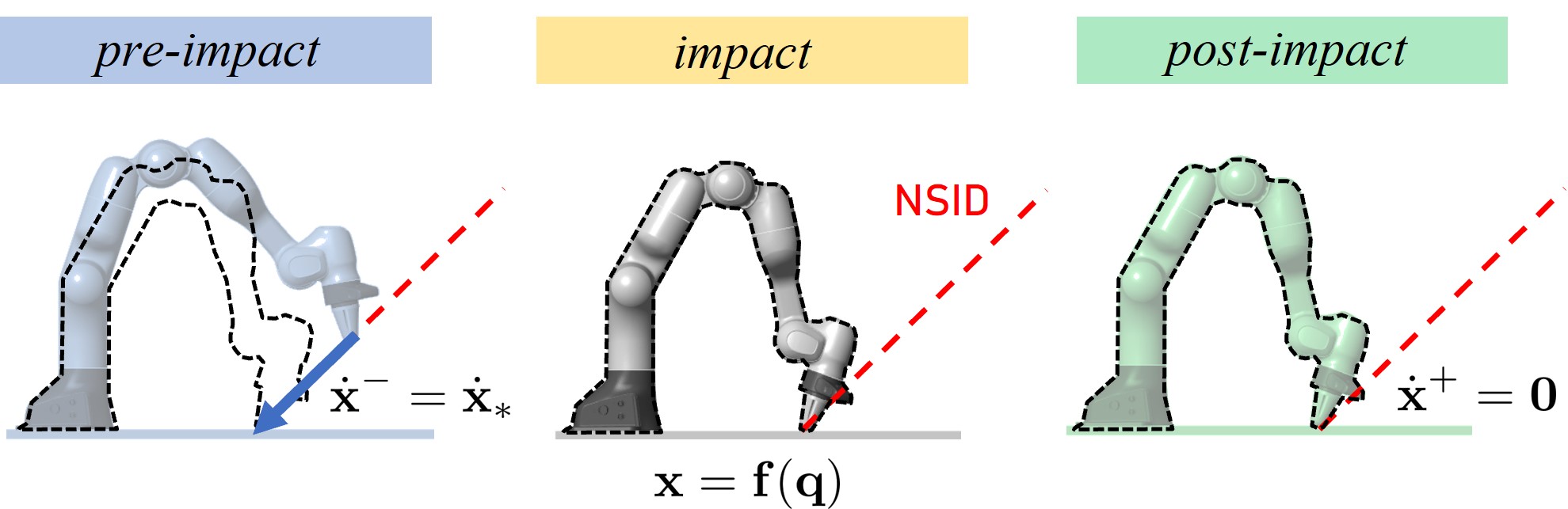}
	\caption{A targeted non-slippage impact of $ \x = \T f(\q) $ is obtained for a fully inelastic contact if $ \dxminus $ aligns the NSID \cite{Kirner2024}.}
	\label{fig:non slippage concept}
\end{figure}

\note{Intuitively, post-impact slippage along a constraint surface with $ e = 0 $ strongly depends on frictional properties of the contact pairing. However, all previous results, and thus also the NSID, have been derived under \Cref{a:frictionless} of a frictionless impact. Thus it is unclear, if the interpretation of the NSID as a "non-slippage impact direction" is also valid for contacts with significant friction. The  derivation in \cite{Kirner2024} assumes a frictionless impact, as well. There it was argued intuitively and without proof that if the NSID solves the targeted non-slippage impact problem for the worst-case scenario of a frictionless contact, slippage will also be prevented in the less critical, frictional case. In the following, we provide a theoretical analysis proving this statement. The result yields a refined interpretation of the NSID for fully inelastic contacts with arbitrary frictional properties, which we will formulate in \Cref{thm:perpendicular force}.}

\subsection{\note{Frictional, Fully Inelastic Impacts}}\label{sec:friction:model}
\note{For the remainder of this section, we assume $ e = 0 $. Further, we discard \Cref{a:frictionless} and consider contact friction acting on the constraint surface in $ l $ tangential directions.\footnote{\note{Typically, it holds $ l= 2 $. A planar case study can be obtained for $l = 1$.}} We assume that all $l$ directions are contained in the task space. The generalized contact force is then defined as
	\begin{equation}\label{eq: A fric}
		\T F_{\textup{contact}} = \begin{bmatrix}
			\Ax \\ \Axfr
		\end{bmatrix}^T \begin{bmatrix}
			\lambda \\ \T \lambda_t
		\end{bmatrix} = \Axtotal^T \T \ltotal.
	\end{equation}
	It consists of the contact force $ \lambda $ normal to the constraint surface and the tangential friction forces $ \lfr\in\Real^l $ yielding a force vector $ \ltotal \in \Real^{1+l} $. The Jacobian $ {\Axtotal\in\Real^{(1+l)\times m}} $ stacks the task space constraint Jacobian $ \Ax $ and the task space friction Jacobian $ \Axfr\in\Real^{l\times m} $, where $ \Axfr $ maps joint velocities to the end-effector velocity components tangential to the contact surface. For the generalized contact torque, it holds $ \T\tau_{\textup{contact}} = \J^T\Axtotal^T\ltotal = \Atotal^T\ltotal,$ with $ \Atotal \in \Real^{(1+l)\times n}. $}

\note{Under an impact of the frictional constraint surface, an impulsive force $ \bar{\T\Lambda} \in \Real^{l+1} $ defined as $ {\bar{\T \Lambda} = \begin{bmatrix}
		\Lambda & \Lfr
	\end{bmatrix}^T = \int_{0}^{\Delta t} \T \ltotal dt} $ with $ {\Delta t \rightarrow 0} $ is obtained. It contains a component $ \Lfr \in \Real^l $ tangential to the constraint surface on top of the previously discussed component $ \Lambda $ normal to the constraint surface. }

\subsubsection{\note{Friction Model}}
\note{We utilize Coulomb's friction model to model dry friction acting on the surface\footnote{\note{Combining Newton's restitution law with Coulombs friction model can create inconsistencies such as apparent energy generation for an arbitrary restitution factor. However, these limitations do not arise for the considered case of fully inelastic impacts \cite{Glocker2012}.}} and pose the following assumption:}\begin{assumption}\label{a:friction cone}
	\note{The Coulomb friction cone defined for contact forces also applies on impulsive forces for the considered fully inelastic case. In particular, it holds
	\begin{equation}\label{eq:impulsive cone}
		\vert\vert \Lfr \vert\vert  \leq \mu_s \Lambda
	\end{equation}
	for a sticking contact $ \Axfr\dxplus = \T 0 $ (cf., e.g., \cite{Smith1991, Brach1998}). Therein, $ \mu_s \geq 0 $ denotes the static friction coefficient.}
\end{assumption}
\note{This corresponds to a representation of the friction cone averaged over the infinitesimal duration of the impact.}

\subsubsection{\note{Impact Map for a Stictional Contact}}
\note{If \eqref{eq:impulsive cone} is fulfilled, it holds $\Atotal\dqplus = \T 0 $ and the robot can be considered as transferring from a free motion phase to a phase where an $ (1+l) $-dimensional constraint is active. Given the derivations from Sec.~\ref{sec:fundamentals} and \ref{sec:model} it can then straightforwardly be shown that it holds
\begin{equation}\label{key}
	\dqplus = \left(\T{I} - \pseudoInv{\bar{A}}{M}\Atotal\right)\dqminus
\end{equation}
and 
\begin{equation}\label{eq:task model stiction}
	\dxplus = \left(\T{I} - \pseudoInv{\bar{A}}{M_x}_{x}\Axtotal\right)\dxminus
\end{equation}
for the relationships between pre- and post-impact velocities in configuration space and task space, respectively. For $ \bar{\T\Lambda} $, then 
\begin{align}\label{key}
	\bar{\T\Lambda} &= -\inv{\left(\Atotal\inv{\M}\Atotal^T\right)}\Atotal\dqminus \\
	\label{eq:Lambda fric x}
	&= -\inv{\left(\Axtotal\inv{\M_x}\Axtotal^T\right)}\Axtotal\dxminus
\end{align}
holds.}

\subsection{\note{Solution of the Frictional Targeted Non-Slippage Impact Problem}}
\begin{theorem}[\textbf{\note{NSID for fully inelastic, frictional contacts}}]\label{thm:perpendicular force}
	\note{Consider a robotic system \eqref{eq:RD rigid} performing an instantaneous, fully inelastic, frictional impact for which \Cref{a:friction cone} is fulfilled. Among all the approach directions in task space yielding a non-slippage impact, the NSID is the unique direction that corresponds to an impulsive force perpendicular to the surface.}
\end{theorem}
\begin{proof}
	Assume a non-slippage impact shall be obtained at configuration $ \q $ for a frictional contact. As we assumed the tangential directions of the constraint surface to be contained in the task space, we thus search for $ \dxminus $ which \begin{itemize}
		\item result in an impulsive force satisfying \eqref{eq:impulsive cone} and
		\item are mapped to $ \dxplus = \T 0 $ by \eqref{eq:task model stiction}.
	\end{itemize}
	Given \eqref{eq:task model stiction}, in general, only pre-impact velocities of the form 
	\begin{equation}\label{eq:dxminus nonslip fric}
		\dxminus= \inv{\M_x}\Axtotal^T\begin{bmatrix}
			\nu \\\T{ p}
		\end{bmatrix}
	\end{equation}
	qualify as candidates rendering a targeted non-slippage impact, with weighting vector $ \T{p}\in\Real^l $. These result in $ \dxplus = \T 0 $, if the cone \eqref{eq:impulsive cone} is fulfilled, as well. Consider an approach under \eqref{eq:dxminus nonslip fric}, with a given, bounded $ \T p $. We aim to determine the minimal static friction coefficient $ \mu_s $ required, such that $ \T{\Lambda} $ fulfills the cone \eqref{eq:impulsive cone} and the end-effector remains in contact with the surface without translational slippage. In this limiting case, $ \T\Lambda $ can be determined from \eqref{eq:Lambda fric x} and lies exactly on the boundary of the cone \eqref{eq:impulsive cone}.
	Plugging \eqref{eq:dxminus nonslip fric} in \eqref{eq:Lambda fric x}, we obtain
	\begin{align}\label{key}
		\T\Lambda = \begin{bmatrix}
			\Lambda \\ \Lfr
		\end{bmatrix} = \begin{bmatrix}
			-\nu \\ \T{p}
		\end{bmatrix}.
	\end{align}
	This yields $ \mu_s = -\vert\vert\T p \vert \vert / \nu $ as the minimum necessary static friction coefficient to obtain $ \dxplus = \T 0 $. 
	
	An approach under the NSID, i.e. $ \T p  = \T 0 $, thus results in an impulsive force perpendicular to the constraint surface. It always fulfills the cone \eqref{eq:impulsive cone} for arbitrary $ \mu_s $. Conversely, all other candidate $ \dxminus $ come with a tangential component $ {\Lfr = \T p \neq \T 0} $ of the impulsive force.
\end{proof}
From the analysis in the proof we can further deduce the following.
\begin{corollary}\label{thm:NSID_cone}
	All $ \dxminus $ solving the targeted non-slippage impact problem of a frictional contact surface with known stiction constant $ \mu_s \geq 0 $ can be expressed as \eqref{eq:dxminus nonslip fric} with weighting vector $ \T p $ satisfying $  \vert\vert\T p \vert\vert \leq -\mu_s \nu $. 
\end{corollary}
\note{Thus, for a frictional surface, the solutions \eqref{eq:dxminus nonslip fric} of the targeted nonslippage impact problem form a cone of approach velocities containing the NSID for $ \T p = \T 0 $ (cf. Fig.~\ref{fig:cone concept}\footnote{\note{The scalings of the depicted $ \dxminus $ are chosen such that they correspond to the same value of $ \Lambda $ normal to the surface.  Only translational directions are displayed as they can be visualized easily. However, the cone belongs to an $ m $-dimensional space in general.}}). Each velocity therein corresponds to an impulsive force contained in the impulsive cone \eqref{eq:impulsive cone}.}
\begin{figure}
	\centering
	\includegraphics[width = 0.4\textwidth]{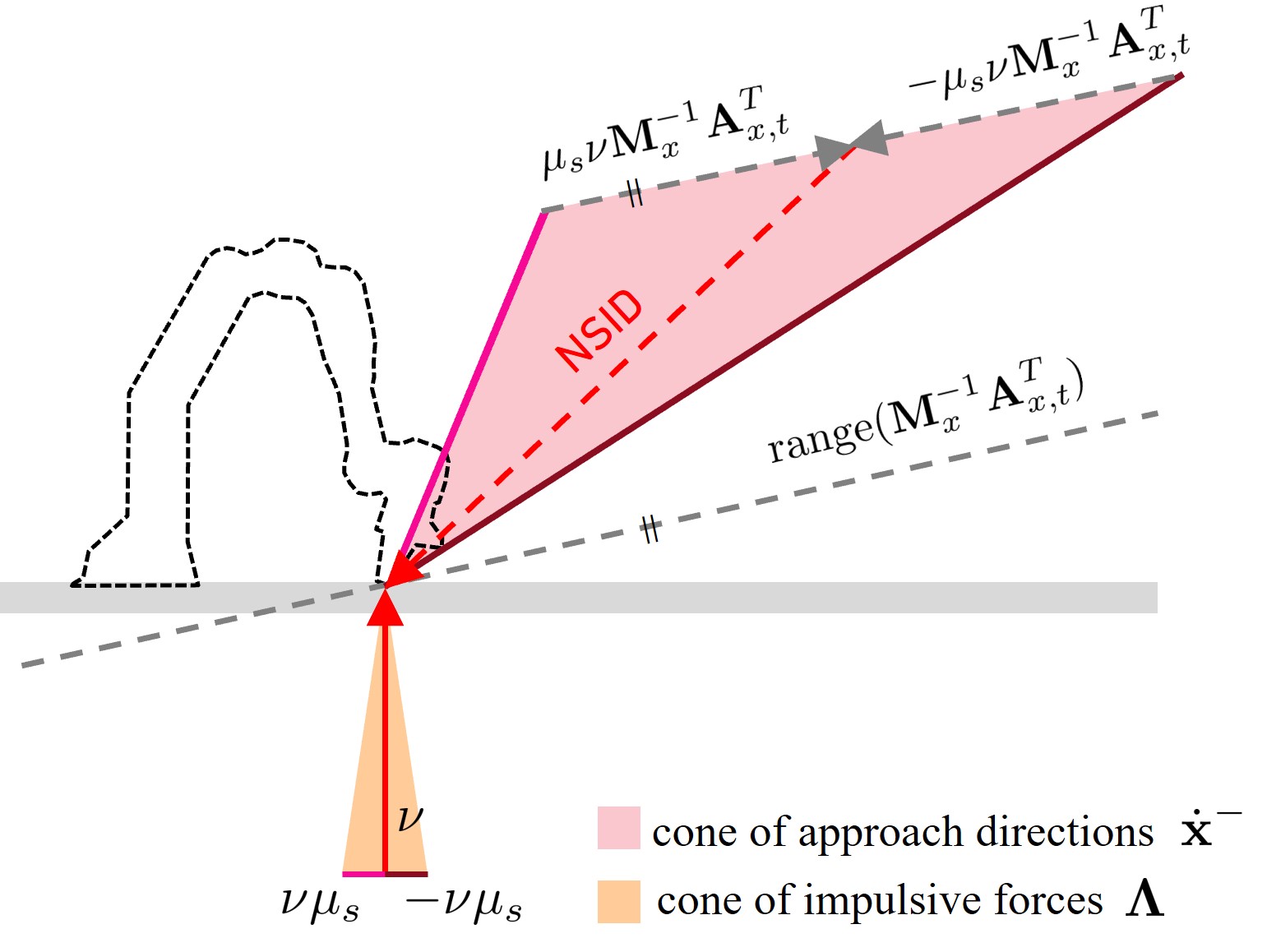}
	\caption{Under contact stiction with static friction coefficient $ \mu_s $, translational approach velocities yielding a targeted non-slippage impact are contained in the shaded, red area.}
	\label{fig:cone concept}
\end{figure}

\note{\textit{Remarks:} We interpreted the impact yielding post-impact stiction as an impact with $ (1+l) $-dimensional constraint. In that case, a $ (1+l) $-dimensional nonsmooth space exists, which is spanned by the NSID, as derived for the frictionless case, and $ \range(\inv{\M_x}\Axfr^T) $ (cf. Fig.~\ref{fig:cone concept}). However, the impulsive cone \cite{Kirner2024} needs to remain fulfilled thus restricting the viable contributions from $ \range(\inv{\M_x}\Axfr^T) $ w.r.t. the NSID contribution. Further, note that all derivations can be performed analogously in configuration space, resulting in a cone of generalized velocities containing the NSID.}

%% file: extended_models_V2.tex
\section{\note{NSID of Flexible Joint Robots}}\label{sec:el joints}
As discussed in Sec.~\ref{sec:fundamentals}, the impact model \eqref{eq: el impact model rigid} has been validated for torque-controlled robots \cite{Aouaj2020,Arias2024}, i.e., flexible-joint robots with relatively stiff joints. While all previous results apply directly to this class of robotic systems, robots with lower joint stiffness are not addressed. However, these robotic systems are expected to withstand relatively harsh impacts, as the elastic elements in the joints protect the motors and gearboxes from impulsive forces to some extent. Therefore, we will examine flexible-joint robots with arbitrary joint stiffness in detail.

\subsection{Dynamics model}
Consider the complete model
\begin{align}\label{eq:RD full}
	\nonumber
	\begin{bmatrix}
		\Ml(\ql) & 	\M_{lm}(\ql) \\ \M^T_{lm}(\ql) & \Mm
	\end{bmatrix} \begin{bmatrix}
	\ddq_l \\ \ddq_m
\end{bmatrix} + 
\note{\T{h}(\ql,\dq)}  = \\
\left(\frac{\partial U(\q)}{\partial \q}\right)^T + \begin{bmatrix}
\T \tau_{\textup{contact}} \\ \T\tau_m
\end{bmatrix}
\end{align}
for the dynamics of a robotic system with flexible joints \cite{DeLuca2008a}. The generalized coordinate vector $ \q\in{\Real}^{(n+r)} $ stacks the coordinates $ \ql\in\Real^n $ belonging to the link-side, and the motor positions $ \qm\in\Real^r $ of the $ r \geq 1 $ motors. The link-side dynamics with link-side inertia matrix $ \Ml\in\Real^{n\times n} $ is coupled to the motor-side dynamics via generalized elastic torques $ \left(\frac{\partial U(\q)}{\partial \q}\right)^T $ with elastic potential function $ U(\q) $, and via inertial couplings, where we denote the inertial coupling matrix as $ \M_{lm}\in\Real^{n\times r} $.
The matrix $ {\Mm \in \Real^{r\times r}} $ denotes the constant diagonal rotor inertia matrix. The motor torques are stacked in the vector $ \T\tau_m \in \Real^r$.

The inertial coupling via $ \M_{lm} $ appears, if the link-side movement of the robot induces a rotation of some motor boxes around their rotor axis \cite{DeLuca2008a}. In that case, the rotational kinetic energy of the rotor depends on both its own spinning with respect to the motor box and the angular velocity of the motor box induced by $ \dq_l $. An example is displayed in Fig.~\ref{fig:Mc}. For large gear reductions, it is often assumed that the inertial coupling is negligible. The model \eqref{eq:RD full} with $ \M_{lm} = \T 0 $ is then referred to as \textit{reduced model} \cite{DeLuca2008a, Spong1987}.

\begin{figure}
	\centering
	\includegraphics[width = 0.4\textwidth]{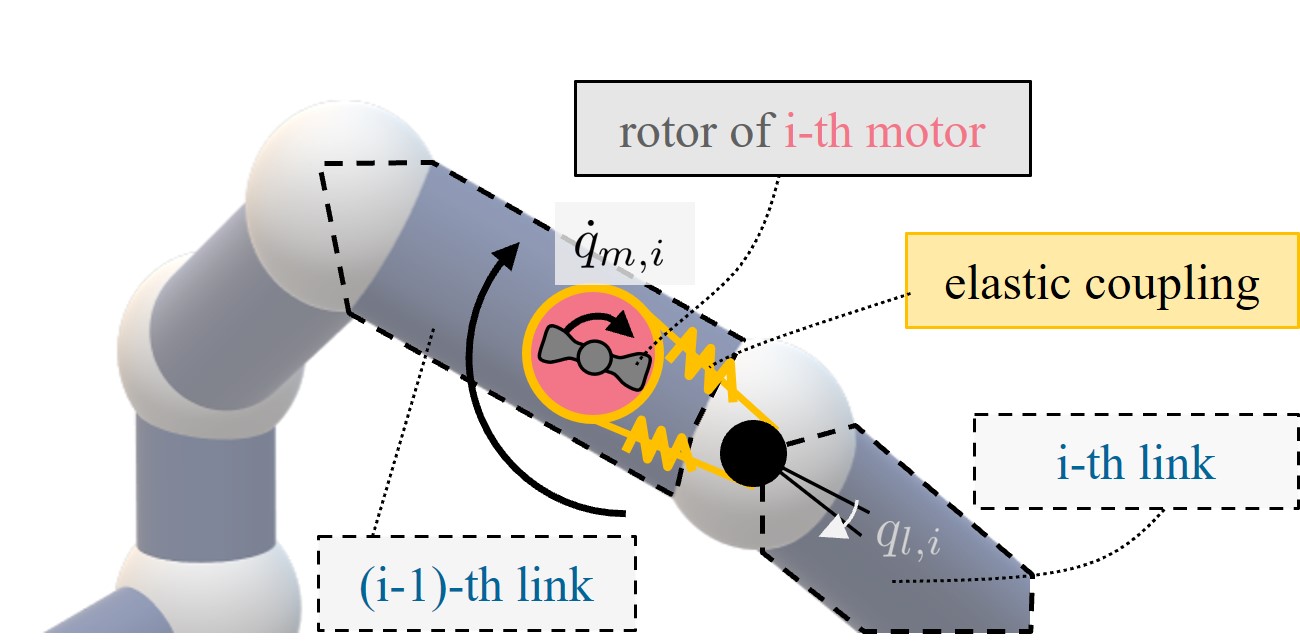}
	\caption{Conceptual sketch of a system where an inertial coupling $ \M_{lm} \neq \T 0 $ exists. The movement of the (i-1)-th link may induce a motion of the i-th motor box around its axis of rotation. Then, the rotational kinetic energy of the rotor also depends on this link-side movement.}
	\label{fig:Mc}
\end{figure}

The task coordinate $ \x = \T f(\ql) $, and thus also the inequality constraint imposed on the task coordinate as defined in Sec.~\ref{sec:fundamentals}, are functions of the link-side coordinates $ \ql $ only. In the following, we omit the last $ r $ vanishing columns of the task and constraint Jacobian and use $ \J(\ql) = \frac{\partial \T f(\ql)}{\partial \ql}$ and $ \A(\ql) = \frac{\partial \phi(\ql)}{\partial \ql}$, respectively, with a slight abuse of notation.

\subsection{General Impact Model}
Integrating \eqref{eq:RD full} over the infinitesimal impact duration, the impact equations \begin{subequations}\label{eq:impact_eq elastic}
	\begin{align}
		\Ml\left(\dqplus_l -\dqminus_l\right) +\M_{lm}\left(\dqplus_m-\dqminus_m\right) &= \T A^T\Lambda\\
		\M^T_{lm}\left(\dqplus_l -\dqminus_l\right) + \Mm\left(\dqplus_m-\dqminus_m\right)  &= \T 0
	\end{align}
\end{subequations}
are obtained for a frictionless impact \cite{Brogliato1999}. Combining them with the restitution model from \Cref{a:restitution}, the impact model for flexible joint robots is straightforwardly derived as
\begin{subequations}\label{eq: el impact model el joint}
	\begin{align}\label{eq:el joint impact link}
		\dqplus_l &= \left(\T I - \left(1+e\right)\pseudoInv{\A}{\bar\M}\A \right) \dqminus_l\\
		\label{eq:el joint impact motor}
		\dqplus_m &= \dqminus_m + \left(1+e\right) \Mm^{-1}\M_{lm}^T\pseudoInv{\A}{\bar\M}\A\dqminus_l
	\end{align}
\end{subequations}
following \cite{Brogliato1999}. 
The matrix $ {\bar{\M}(\ql) \in \Real^{n\times n}}$ is defined as \begin{equation}\label{eq:Mbar}
	\bar\M(\ql) = \Ml(\ql) - \M_{lm}(\ql)\inv{\Mm}\M_{lm}^T(\ql).
\end{equation}

\subsection{Nonsmooth and Impact Invariant Velocities}
\note{Given the impact model \eqref{eq: el impact model el joint}, nonsmooth and impact invariant velocities can be identified for the link and motor side.}
\subsubsection{\note{Link side}}
\begin{lemma}[\textbf{Link-side NSID and impact-invariant velocities}]\label{cor:el joint iv}
	Consider a flexible joint robotic system \eqref{eq:RD full} undergoing an impact satisfying Assumptions~\ref{a:non singular basic}-\ref{a:restitution}. All results from Sec.~\ref{sec:model} can be applied to the link-side impact behavior using $ {\M = \bar{\M}} $, $ \A(\ql) = \frac{\partial \phi(\ql)}{\partial \ql} $, and $ \J(\ql) = \frac{\partial \T f(\ql)}{\partial \ql}$. In particular, a NSID in configuration and task space can be calculated as defined in Theorems~\ref{thm:dqimp}~and~\ref{thm:dximp}. Impact invariant velocities can be obtained from Theorems~\ref{thm:dqiv}~and~\ref{thm:task invariant}.
\end{lemma}
We denote the link-side NSID as $ \dq_{l,\ast} $ in the following.

\subsubsection{Nonsmooth Motor Velocities}
\begin{figure*}
	\centering
	\includegraphics[width = 0.95\textwidth]{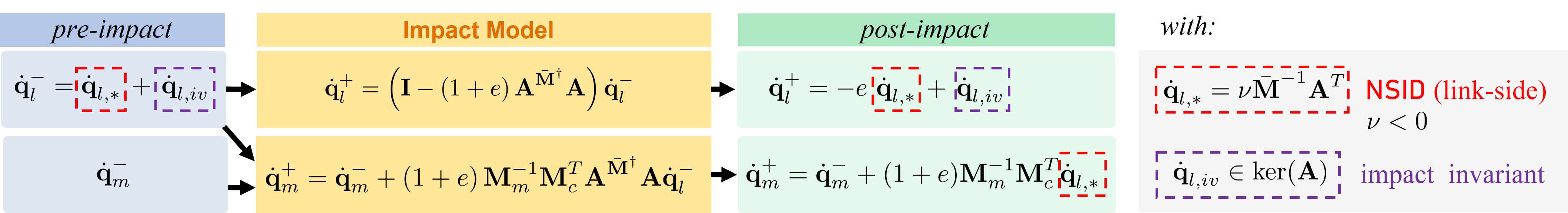}
	\caption{Impact model of a flexible joint robot modeled via \eqref{eq:RD full} undergoing an impact. Besides the link-side velocities, also the motor velocities may become nonsmooth. The amount of the motor velocity jump is then proportional to the NSID contribution $ \dq_{l,\ast} $ to the link-side pre-impact velocity $ \dqminus_l $.}
	\label{fig:el joint concept}
\end{figure*}
While flexible joint robots may seem well suited for impact tasks, \eqref{eq:el joint impact motor} reveals, that, in general, they do not fully protect the motors from impulsive torques \cite{Brogliato1999}. One obtains
\begin{equation}\label{key}
	\Mm\left(\dqplus_m-\dqminus_m\right)  = - \M_{lm}^T\inv{\bar{\M}}\A^T\Lambda,
\end{equation}
where an impulsive torque $ \T T = - \M_{lm}^T\inv{\bar{\M}}\A^T\Lambda \in \Real^r $ is acting on the motor side, inducing a motor velocity jump. From \eqref{eq:el joint impact motor} it becomes immediately clear that the jump only depends on pre-impact link-side velocities $ \dqminus_l $ and not on pre-impact motor velocities $ \dqminus_m $. This is due to the fact that the inequality constraint only restricts link-side motions. 

\begin{theorem}[\textbf{nonsmooth motor velocities}]\label{thm: motor side}
	Consider a flexible joint robotic system \eqref{eq:RD full} undergoing an impact satisfying Assumptions~\ref{a:non singular basic}-\ref{a:restitution}. If it holds $ {\M_{lm}^T\dq_{l,\ast} \neq \T 0} $, i.e. if the NSID $ {\dq_{l,\ast} = \nu\inv{\bar{\M}}\A^T} $ in the configuration space is not in the null space of the inertial coupling matrix $ \M_{lm}^T $, the impact induces nonsmooth motor velocities. The discontinuity solely depends on the NSID contribution to the pre-impact link-side velocity.
\end{theorem}

\begin{proof}
	Given \Cref{cor:el joint iv}, evaluating \eqref{eq:el joint impact motor} for the general pre-impact velocity $ {\dqminus_l = \dq_{l,\ast} + \dq_{l,\textup{iv}}} $ yields 
	\begin{align}
		\dqplus_m & = \dqminus_m + \left(1+e\right) \Mm^{-1}\M_{lm}^T\dq_{l,\ast}\\
		\label{eq: dqplus 2}
		& = \dqminus_m + \left(1+e\right)\nu \Mm^{-1}\M_{lm}^T\inv{\bar{\M}}\A^T.
	\end{align}
	The second term therein only depends on the NSID contribution and is nonzero iff $ \M_{lm}^T\dq_{l,\ast} \neq \T 0 $.
\end{proof}
Note that plugging $ \A = \Ax\J $ and $ \dxminus = \dx\imp +\dx\iv $ in \eqref{eq:el joint impact motor} and using \eqref{eq: dximp}, one can reconstruct \eqref{eq: dqplus 2}. Thus, the post-impact motor velocities can also directly be predicted from $ \dxminus $.

According to \Cref{thm: motor side}, an impact of a flexible joint robotic system with inertial couplings in a configuration for which $ {\M_{lm}^T\dq_{l,\ast} \neq \T 0} $ is fulfilled always yields nonsmooth motor velocities. The practical relevance of this effect depends on the specific robotic system. When building or choosing flexible joint robotic hardware for impact applications, the possible harmful effects of inertial couplings on motors should be evaluated. We summarize results of this section with respect to the link-side generalized velocities in Fig.~\ref{fig:el joint concept}.

\section{\note{NSID of Constrained Robotic Systems}}\label{sec:constraint}
Consider the robotic system \eqref{eq:RD rigid} \note{or the flexible joint robot \eqref{eq:RD full}}. Besides the inequality constraint \eqref{eq:constraint}, we now assume it is also subject to a {$ p $-dimensional} equality constraint $ \T \phi\eq(\q) = \T 0 $ on the generalized coordinate $ \q $. Thus, the generalized contact torque 
\begin{equation}\label{eq:eq constraint tau contact}
	 \T\tau_{\textup{contact}} = \A^T\lambda + \A\eq^T\T \lambda\eq 
\end{equation}
includes a contribution induced by the generalized forces $ {\T\lambda\eq \in\Real^p} $ in direction of the equality constraint. It holds $ {\A\eq = \frac{\partial \T \phi\eq(\q)}{\partial \q}} $ for the Jacobian $ {\A\eq \in \Real^{p\times n}} $ of the equality constraint. In the following, we address situations satisfying the assumption below.
\begin{assumption}\label{a:constraint}
	The equality constraint Jacobian $ \A\eq $ is of full rank. Further, all rows of $ \A\eq $ are linearly independent from $ \A $, i.e. it holds $ \textup{rank}\left(\left[\A\eq^T, \A^T\right]\right)=p+1 $.
\end{assumption}
\note{The following analysis is performed for the robotic system \eqref{eq:RD rigid} with non-elastic joints for the sake of readability. However, all results can be applied mutatis mutandis for flexible joint robots.}

\subsection{\note{Impact Model}} 
\begin{lemma}[\note{Impact Model of a constrained robotic system}]\label{thm:constrained model}
	\note{Consider a robotic system \eqref{eq:RD rigid} subject to a $ p $-dimensional equality constraint $ \T{\phi}_c(\q) = \T 0 $ and the one-dimensional inequality constraint \eqref{eq:constraint}. Under an impact satisfying Assumptions~\ref{a:non singular basic}-\ref{a:constraint}, the post-impact generalized velocities are obtained from the pre-impact state as follows:
	\begin{align}\label{eq:impact map constrained}
		\dqplus =  \left( \T I - \left(1+e\right)\T P_{q,c}\right)  \dqminus.
	\end{align}
	The projector $ \T P_{q,c} $ therein can be formulated as
	\begin{equation}\label{eq:Pqc 2}
		\T P_{q,c} = \T W_{q,c} \A^T\inv{\left(\A\T W_{q,c} \A^T\right)}\A
	\end{equation}
	with $ \T{W}_{q,c} = (\T I -\T P_c)\inv{\M}(\T I -\T P_c^T) \in \Real^{n\times n}$ and $ {\T P _c = \pseudoInv{\A}{\M}_c\A_c} \in \Real^{n\times n} $.}
	\end{lemma}
	\begin{proof}
		\note{To derive the model, the robot dynamics \eqref{eq:RD rigid} with \eqref{eq:eq constraint tau contact} are integrated over the infinitesimal impact. Considering that it always holds $ \dq = (\T I - \T P\eq)\dq $, i.e. all generalized velocities are compatible with the equality constraint at all times, the final model is obtained. We provide the detailed derivation in \Cref{app:constrained_model}.}
	\end{proof}
	\note{ \textit{Remarks:} The inverse in \eqref{eq:Pqc 2} exists due to \Cref{a:constraint}. Further, note that the projector $ \T P_{q,c} $ is in a similar form as the one from \eqref{eq:projector} used throughout this paper. }
\subsection{\note{NSID and Impact Invariant Velocities}}
\begin{theorem}[\textbf{NSID of a constrained system in configuration space}]\label{thm:constraint configuration}
	\note{Consider a constrained robotic system undergoing an impact, modeled by \Cref{thm:constrained model}.} The constrained NSID $ \dq_{*,c} $ in configuration space is obtained as the dynamically consistent projection of the unconstrained NSID $ \dq\imp $ from \Cref{thm:dqimp} on the equality constraint, i.e.
	\begin{equation}\label{key}
		\dq_{*,c} = (\T I -\T P\eq)\dq\imp = \nu(\T I -\T P\eq)\inv{\M}\A^T,
	\end{equation} with the projector $ \T P\eq = \pseudoInv{A}{\M}\eq\A\eq $, and $ \nu < 0 $. \note{The feasible impact-invariant space is $ (n-p-1) $-dimensional.}
\end{theorem}

\begin{proof}
	The NSID is straightforwardly identified from the range of the projector $ \T P_{q,c} $ in \eqref{eq:Pqc 2} as $ {\dq_{*,c} = \nu\T W_{q,c}\A^T} $. By definition of $ \T W_{q,c}$, \eqref{eq: MP} and idempotence of $ (\T I - \T P\eq) $, we obtain $ {\dq_{*,c} = \nu(\T I -\T P\eq)\inv{\M}\A^T = (\T I -\T P\eq)\dq\imp} $. By definition of $ (\T I -\T P\eq) $, this corresponds to a dynamically consistent projection \cite{Khatib1995} of $ \dq\imp $ on the constraint.
	
For the impact-invariant space, we analyze the null space of $ \T{P}_{q,c} $, which is $ (n-1) $-dimensional, in general. However, $ \dqminus $ from the $ p $-dimensional subspace defined as $ {\ker(\T I - \T P\eq)} $ are not compatible with the equality constraint and thus unfeasible. \note{Given \Cref{a:constraint}, the impact-invariant subspace can thus be defined as the orthogonal complement to $ \ker(\T I - \T P\eq)  $ in $ \ker(\A(\T I - \T P\eq)) $. 
Hence, it is ${( n-p-1)} $-dimensional.}
\end{proof}

Utilizing $ \A = \Ax\J $, the task space map
\begin{equation}\label{eq:dxplus constrained}
	\dxplus = \left(\T I - \left(1+e\right) \T{P}_{x,c} \right)\dxminus
\end{equation}
can be directly obtained from \eqref{eq:impact map constrained}, with the projector
\begin{equation}\label{eq:Pxc}
	\T{P}_{x,c} = \T{W}_{x,c}\Ax^T\inv{\left(\Ax\T{W}_{x,c}\Ax^T\right)}\Ax
\end{equation}
and $ \T W_{x,c} = \J(\T I -\T P\eq)\inv{\M}\Peq\J^T $.\footnote{The inverse in \eqref{eq:Pxc} is invertible due to \Cref{a:constraint}.} 
Thus, we conclude the following:
\begin{corollary}\label{thm:constraint task}
	For the task space NSID $ \dx_{\ast,c} $ of a constrained robotic system it holds
	\begin{equation}\label{key}
		\dx_{\ast,c} = \nu\T{W}_{x,c}\Ax^T = \J\dq_{\ast,c} = \J(\T I -\T P\eq)\dq\imp,
	\end{equation}
	with the constrained NSID $ \dq_{\ast,c} $ in configuration space, the unconstrained NSID $ \dq\imp $ in configuration space, and $ \nu < 0 $. The impact invariant task space $ {\ker(\Ax) \cap \range(\J(\T I - \T P\eq))} $ contains all feasible task velocities, which are in the null space of the inequality constraint.
\end{corollary}

%% file: experiments.tex
\section{Experiments \note{with a passive Mechanism}}\label{sec:experiments}
\note{
	The remainder of this work aims at experimentally evaluating the key properties of the NSID. To validate the calculation of the NSID independent of a controller, we start with a passive multibody sytem in this section. In particular, we will evaluate the following result:}
\note{\textit{Only an impact under the NSID aligns the rebound path with the approach path independent of the restitution factor.}}
	
\note{Using the passive system is motivated by the fact that, due to \Cref{a:instantaneous}, the impact model \eqref{eq: el impact model rigid} from \cite{Walker1994} equates the motion of an actuated robot during an impact to that of its underlying passive multibody system. By employing such a passive system for our initial experiments, we demonstrate that (i) the NSID is indeed an inherent property of passive multibody systems, and (ii) the alignment of approach and rebound paths is triggered solely by the approach direction, rather than by motor dynamics or any control strategy.}
\note{We will then go on to analyze impacts with a torque-controlled robot in Sec.~\ref{sec:franka_exp}.}

\subsection{Passive Experimental Setup}\label{sec:exp:passive_setup}
\note{The utilized jointed multi-body system 3D-printed from PLA is depicted in Fig~\ref{fig:exp_setup}. 
It comprises three revolute joints with parallel axes. The joint angles are stacked in the vector $ \q\in\Real^3 $ of generalized coordinates. The motion of the system is restricted to the $ yz $-plane of the world frame, located in the first joint. We utilize the end-effector position $ \x = \begin{bmatrix}
	y & z 
\end{bmatrix}^T \in \Real^2 $ of the center of the spherical tool tip in that frame as a task coordinate. \note{It holds $ l_1 = l_2 = \SI{0.165}{\metre}$ and $ l_3 = \SI{0.128}{\metre} $ for the lengths of the three links and $ m_1 = \SI{0.299}{\kilogram} $, $ m_2 =\SI{0.208}{\kilogram} $ and $ m_3 = \SI{0.272}{\kilogram} $ for their masses.} The mechanism is placed above a horizontal surface, which is impacted with the end-effector in the experiments. To evaluate effects of different contact pairings, we utilize both a coated chipboard and an aluminum profile as a contact surface.}

\begin{figure}
	\centering
	\includegraphics[width = 0.4\textwidth]{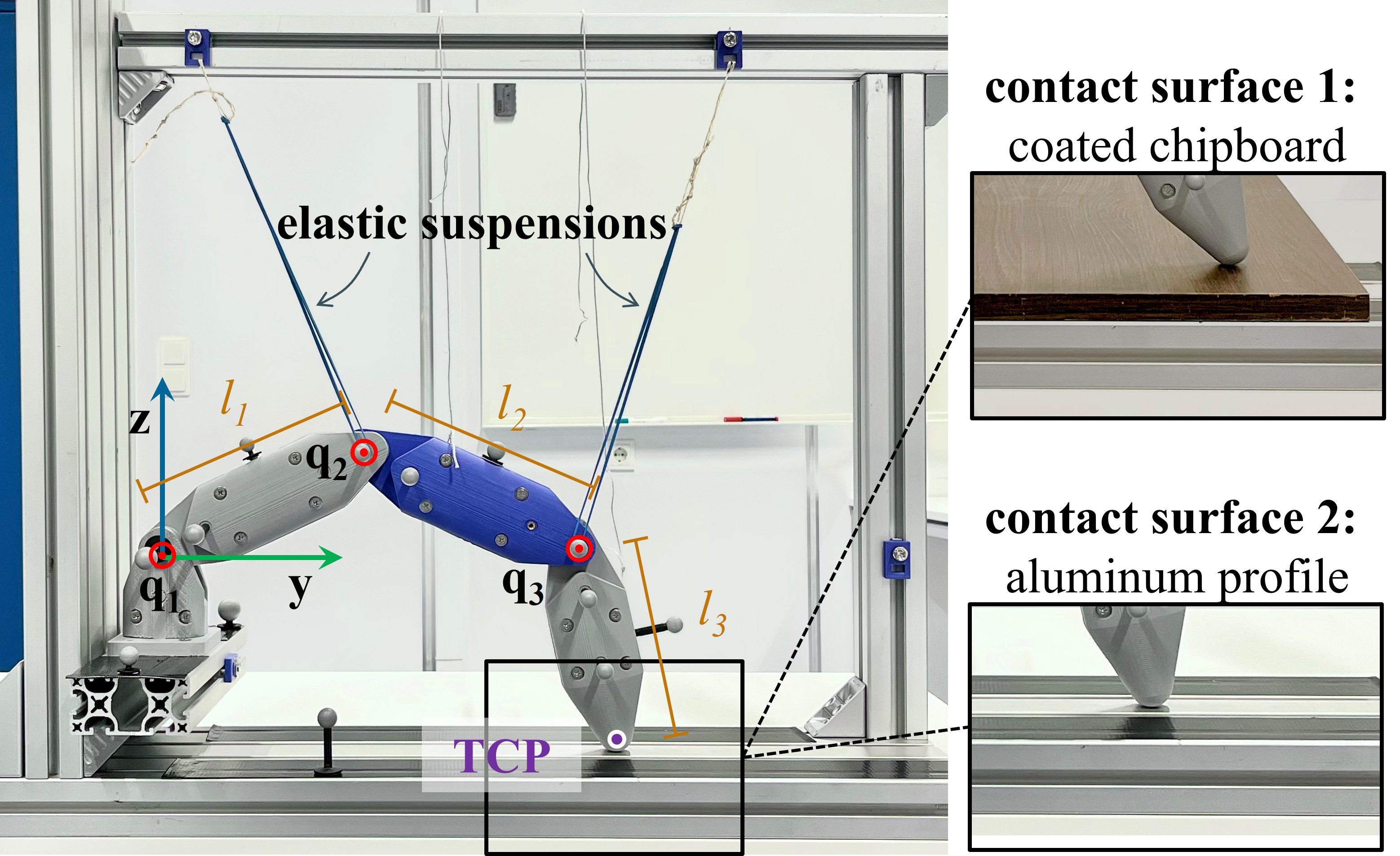}
	\caption{\note{Passive, 3 DoF, jointed multibody system and contact surfaces utilized for the experimental validation.}}
	\label{fig:exp_setup}
\end{figure}

Gravity acts along the negative $ z $-axis of the world frame and is utilized to generate the impact motions. Given elastic suspensions made from soft rubber bands, the system favors elbow-up configurations throughout the experiments. This ensures that only the end-effector collides with the surface and increases repeatability of the tests without affecting the impact dynamics.  

For every trial, the mechanism was moved manually via attached guiding strings, such that the end-effector lifted from the contact surface. Upon release, the mechanism accelerated under the influence of gravity, causing the end-effector to collide with the surface. \note{The motion of each link during the impact tests was recorded using an OptiTrack motion capture system at a frequency of $ \SI{1}{\kilo\hertz} $.} During post-processing, moreover, the time evolution of the joint angles $ \q(t) $ was computed from the data using kinematics of the system. \note{The raw and unfiltered measurement data were employed in the subsequent evaluation so as not to distort the effects of the impact dynamics.}

\subsection{Results}
Before analyzing effects of the approach direction on the rebound direction, let us consider one trial per contact in detail and define a measure for the impact timing.

\subsubsection{Definition of the impact and elastic properties}
\begin{figure}
	\centering
	\includegraphics[width = 0.42\textwidth]{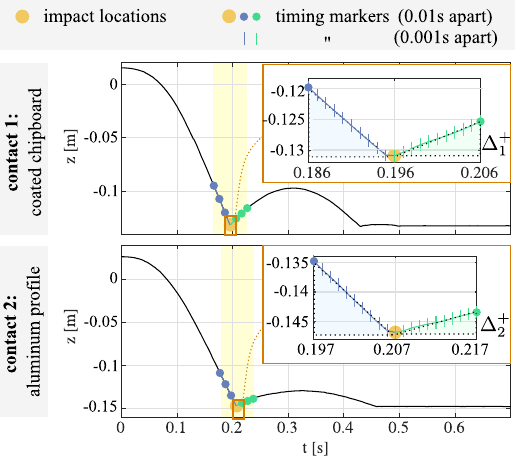}
	\caption{\note{Vertical end-effector position $ z(t) $ over time for one impact experiment on contact 1 and 2 each. We identify the impact duration to be in the range of {1-2ms}, and define the impact as the minimum of z(t) for the first collision. The velocity $ \dot{z} $ is approximately $ \SI{1.2}{\metre/\second} $ before the impact in both cases.}}
	\label{fig:res_t}
\end{figure}
\begin{figure*}[ht]
	\centering
	\includegraphics[width = 0.93\textwidth]{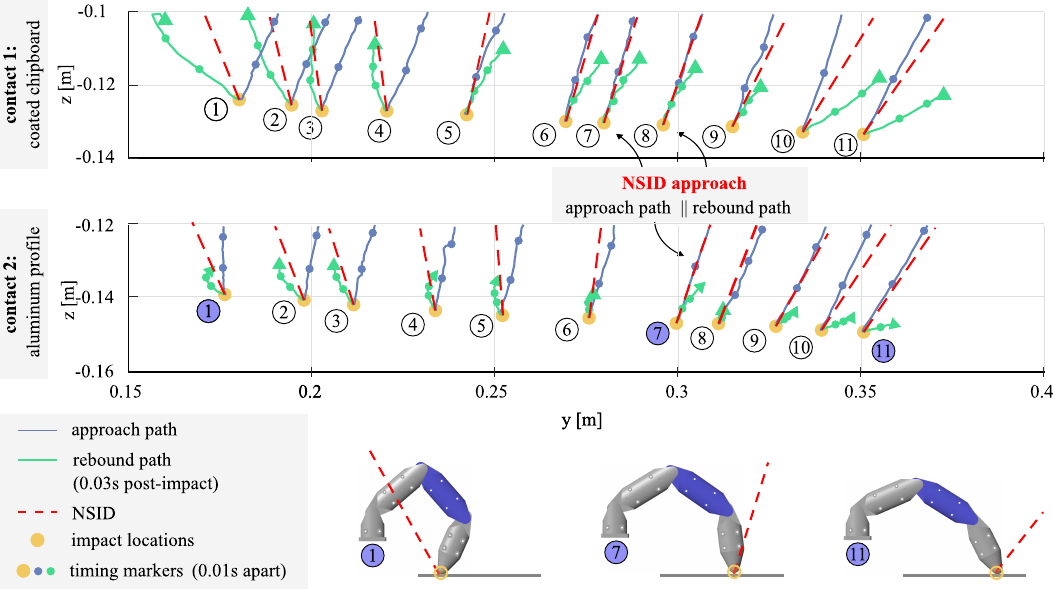}
	\caption{\note{Measured pathes of the translational endeffector position in eleven impact experiments per contact surface as well as visualizations of the configurations at the impact time for three examples. Only if the approach path aligns the NSID, the rebound path aligns with the approach path as well. In all other cases, the angle between NSID and end-effector path changes its sign upon the impact.}}
	\label{fig:res}
\end{figure*}

Figure~\ref{fig:res_t} depicts the evolution of the $ z $-coordinate of the end-effector over time for one exemplary trial per contact surface. We observe a sharp edge at the first minimum comprising 2-3~time \note{steps. This indicates} an impact duration roughly in the range of the sampling time of the motion capture system. 
We will refer to the impact as the local minimum of $ z(t) $ obtained under the first collision of the end-effector with the surface in the following. We will indicate it with a yellow marker in the plots, separating the approach phase from the rebound phase. The analysis presented in this work is valid for infinitesimal times around an impact. However, to provide more context, we will depict time windows of \SI{0.03}{\second} pre- and post-impact, in the following (cf. the sequence shaded in yellow for the two experiments of Fig.~\ref{fig:res_t}).

The selected contact surfaces have distinct elastic properties. While only slight bouncing is observed for the experiments on the aluminum surface, the end-effector clearly rebounds from the coated chipboard. This is confirmed with the zoomed plots of Fig.~\ref{fig:res_t}. While the tangents of the approach path are parallel at the impact for both surfaces, the tangent of the rebound path has a larger inclination for the coated chipboard. With \cref{a:restitution}, this implies that the restitution factor for impacts on the coated chipboard is \note{significantly higher.}

\subsubsection{Varying approach angles with respect to the NSID with a passive mechanism}
\note{We will focus on eleven trials per contact surface in the following. These were selected based on the horizontal location of the impact, such that the same range of impact locations is covered approximately equidistantly. Figure~\ref{fig:res} depicts the corresponding approach and rebound paths of the end-effector in the $ yz $-plane around the first collision. } Waypoints with a timely distance of \SI{0.01}{\second} are marked with dots.

It is important to note that all approach paths (plotted in blue) are roughly parallel. This effect is visible throughout the data set. Thus, the reconfigurations necessary to obtain the different impact locations have negligible effect on the direction of the end-effector path, which is dictated by gravity and the elastic suspensions. However, in our setup, these configuration changes come with a noticeable variation in the NSID in task space. For the presented results, the NSID, as computed at the time of the impact, tilts approximately \SI{60}{\degree} around the negative $ x $-axis. With increasing $ y $-values of the impact location, it reduces its angle to the approach path, crosses it, and then increases the relative angle again. This property is utilized to analyze NSID impacts, as well as impacts under positive and negative angles with respect to the NSID of the passive mechanism in the following.

\subsubsection{Relation between the approach direction, the rebound direction, and the NSID}
Consider the impacts on the aluminum profile depicted in the lower plot of Fig.~\ref{fig:res}. In experiment~7 the approach path in task space is tangential to the NSID in task velocity space. In this case, the rebound path aligns the approach path immediately after the impact, which matches with the presented results. With increasing angle between the NSID and the approach path, also the angular deviation of the rebound path from the approach path directly after the impact increases. We observe that an approach with a positive angle between the NSID and the approach path results in a negative angle between the NSID and the rebound path and vice versa. This is predicted by the presented theory as well: when splitting up the pre-impact velocities in an NSID contribution and an impact-invariant contribution following \eqref{eq:dxminus decomposed}, the impact-invariant contribution, which corresponds to horizontal motions here, is preserved over the impact (cf. Fig.~\ref{fig:X}). \note{Experiment 9 appears to be an outlier. Eventhough the impact seems to have occured under the NSID, a deflection of the rebound path is observed. This could, e.g. be due to measurement inaccuracies or backlash in the proposed mechanism.}

Now consider impacts on the coated chipboard as reported in the upper plot of Fig.~\ref{fig:res}. First, note that the rebound paths, which are always depicted for \SI{0.03}{\second}, are longer than for the impacts on the aluminum surface. This aligns with the earlier observation that the impacts on the coated chipboard are more elastic. Nevertheless, the same conclusions can be drawn from the experiments. The approach path roughly aligns the NSID in experiments~7 and 8 implying a rebound along the approach path. In all other cases, the rebound direction deviates from the approach direction in the expected manner.

\subsection{Conclusion}
\note{The experiments confirm that impacts under the NSID align the rebound path with the approach path for impacts with arbitrary and possibly unknown restitution factors. In utilizing a passive mechanism as a minimal example, we highlight that the NSID is a physical property of a multi-body system.}

\section{\note{Experiments with a torque-controlled robot}}\label{sec:franka_exp}
\note{In the following, we will focus on experiments with a torque-controlled robot. Besides evaluating rebound behaviour under various approach directions and controllers, we will also analyze contact force measurements in presence of contact friction.}

\subsection{\note{Experimental Setup}}\label{sec:franka_setup}
\note{We utilize a Franka robot (FR3) under various control modes, approach directions and endeffectors.}
\begin{figure}
	\centering
	\includegraphics[width = 0.48\textwidth]{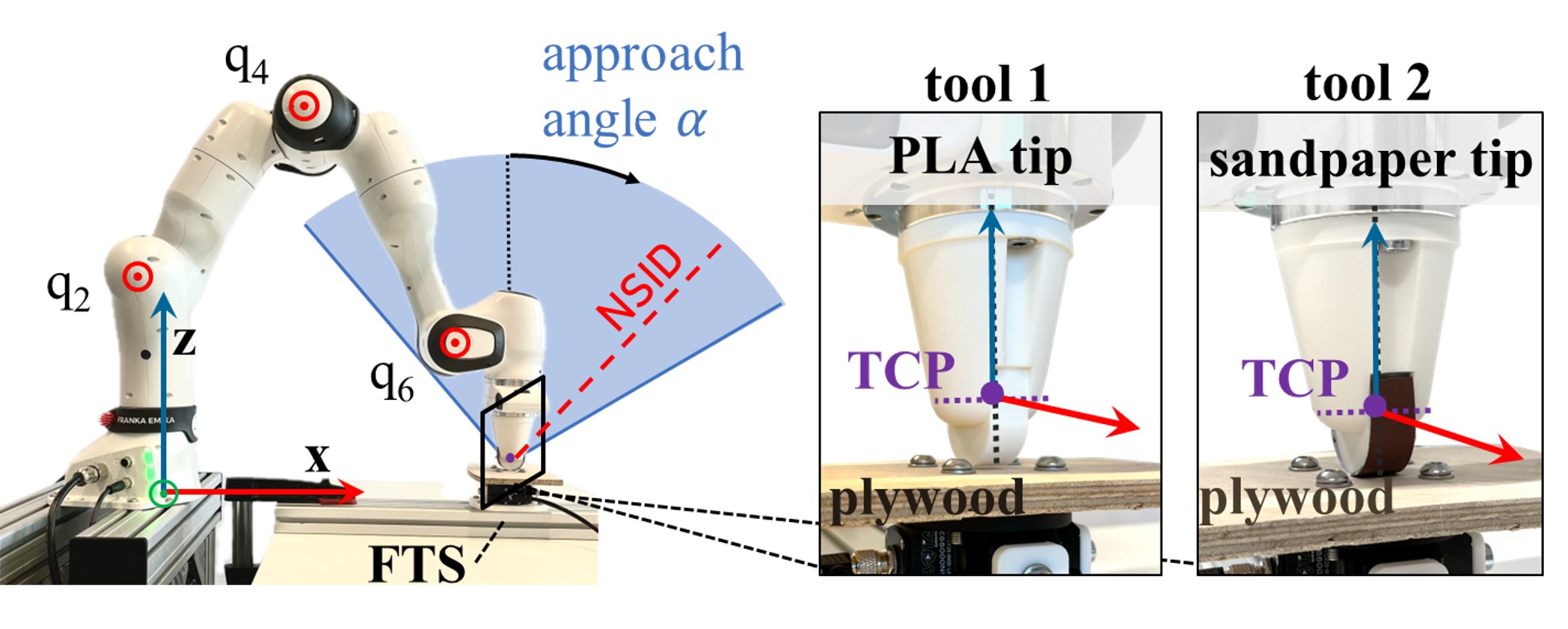}
	\caption{\note{Hardware setup for the experiments with a torque-controlled robot. The FR3 impacts on a plywood, mounted on a FTS. Two tools with identical geometry and inertial properties but distinct coating are utilized.}}
	\label{fig:exp_setup_franka}
\end{figure}
\subsubsection{\note{Hardware Setup}}
\note{The experimental setup is depicted in Fig.~\ref{fig:exp_setup_franka}. The robot is placed in front of a horizontal plywood screwed on a SensONE 6-axis force-torque sensor (FTS) by Bota Systems AG. In order to allow for clearer graphical representation of the end-effector paths, we restrict the experiments to planar motions. Hence, we only consider joints 2, 4, and 6 for the following analysis. A stiff joint PD-controller regulates the remaining joints in their initial position.}

\note{We utilize two tools attached to the robot's flange. They are 3D printed from PLA, and share the same geometry, and inertial properties. However, we attached sandpaper on the tip of tool~2 (cf. Fig.~\ref{fig:exp_setup_franka}). We define the TCP at the center of the cylindrical tip. The task coordinate $ \x $ stacks the translational position $ \begin{bmatrix}
		x & z
	\end{bmatrix}^T $ of the TCP and its orientation, in the following.
}

\subsubsection{\note{Impact Configuration}}
\note{All experiments are performed at the desired impact configuration $ \q_{\textup{imp}} $ depicted in Fig.~\ref{fig:exp_setup_franka}.  It holds $ {\x_{\textup{imp}} = \T{f}(\q_{\textup{imp}}) = \begin{bmatrix}
		\SI{0.546}{\metre} & \SI{0.07}{\metre} & \SI{0}{\radian}
	\end{bmatrix}^T}$. The NSID is evaluated as $ {\dx\imp(\q_{\textup{imp}},\nu) \approx \nu\begin{bmatrix}
		0.15 & 0.16 & -0.54
	\end{bmatrix}^T}$, which corresponds to a translational approach angle of ${\alpha \approx 43.5^\circ}$ (cf. Fig.~\ref{fig:exp_setup_franka}). The configuration has been chosen such that impacts under a broad range of approach angles can be performed without triggering safety thresholds of the FR3. The FTS is placed, such that the tool tip contacts the plywood above its center.
}

\subsubsection{\note{Approach Trajectories}}
\note{The experiments are performed using straight, translational taskspace paths. The approach paths lead toward the impact configuration, and the return paths start from it. Each path is uniquely defined by its (approach) angle~$\alpha$ (cf. Fig.~\ref{fig:exp_setup_franka}), which we vary within the range $\alpha \in [-40^\circ, 60^\circ]$.}
\note{A time parametrization is chosen, such that all impacts occur at a constant desired normal contact velocity of $ \dot{\varphi}_d = \SI{0.1}{\metre/\second}  $. The endeffector orientation is kept constant.}

\subsubsection{\note{Control Modes}}
\note{The robot is operated in three control modes.}
\note{The first one is a \textit{PD+ type Cartesian compliance controller}. We apply the control law
	\begin{equation}\label{eq:PD}
		\T \tau = \J^T\left(\M_x\ddx_d  + \T\mu\dx_d - \T{D}_{PD_+}\dot{\tilde{\mathbf{x}}} - \T{K}_{PD_+}\tilde{\x} \right) +\T g,
	\end{equation}
	with the position $ \tilde{\x} = \x - \x_d $ and velocity errors $ \dot{\tilde{\mathbf{x}}} = \dx - \dx_d $. Furthermore, it holds
	$ \T\mu(\q,\dq) = \M_x\left(\J\inv{\M}_x\T{C}-\dot{\J}\right)\inv{\J} $, with the Coriolis- and centrifugal matrix $ \T{C}(\q,\dq) \in \Real^{n\times n}$. The gravitational torques are stacked in $ \T g(\q) \in \Real^n$. The gain matrix 
	$ \T K_{PD_+} \in \Real^{3\times 3}$ is diagonal with entries $ \begin{bmatrix}
		\SI{1500}{\newton/\metre}  & \SI{1500}{\newton/\metre} & \SI{150}{\newton\metre/\radian}
	\end{bmatrix} $. A state dependent damping is chosen based on double-diagonalization \cite{AlbuSchäffer2003} with damping factor $ \xi = 0.65 $.}

\note{Secondly, a \textit{computed torque (CT) controller} in task space is employed with control law
	\begin{equation}\label{eq:CT}
		\T \tau = \J^T\M_x\left(\ddx_d  + \T\mu\dq - \T{D}_{CT}\dot{\tilde{\mathbf{x}}} - \T{K}_{CT}\tilde{\x} \right) +\T g.
	\end{equation} The diagonal gain matrices are chosen as $ \T K_{CT} = \operatorname{diag}(\SI{650}{\newton/\metre} ,\SI{650}{\newton/\metre},\SI{1000}{\newton\metre/\radian}) $ and $ {\T D_{CT} = 2\xi\operatorname{diag}(\sqrt{650},\sqrt{650},\sqrt{1000})} $, respectively.}

\note{Finally, we also utilize the robot in \textit{gravity compensation} mode, i.e. $ \T \tau = \T g $.
}

\subsubsection{\note{Impact Detection}}
\note{We utilize an online impact detection strategy to switch between pre- and post-impact trajectories and control paths. It employs a customly tuned threshold for the vertical acceleration.
}

\subsection{\note{Relation between Approach and Rebound Paths}}\label{sec:franka_pathes}
\note{Given our theoretic analysis and the experimental results with the passive system, we pose the following hypothesis to be evaluated:}
\note{\textit{Only for an impact under the NSID, an alignment of the approach path with the rebound path immediately after the impact can be achieved for a torque-controlled robot, independent of the applied rigid-body controller.}}

\subsubsection{\note{Experiments}} \note{We evaluate impacts with the PLA tool~1 under three approach angles ${ \alpha \in \{0^\circ, 20^\circ, 43.5^\circ\}}$, where the last one corresponds to the NSID. For every task space tracking controller (i.e. the PD+ controller \eqref{eq:PD} and CT controller \eqref{eq:CT}), and approach direction, we perform two experiments. They are distinct in the post-impact strategy. For the "\textit{move back}" experiments, the trajectory is reversed upon impact detection. Thus, ideally, the end-effector moves in and out of the contact along the same path. Conversely, for the "\textit{free motion}" experiments, the control mode is switched to gravity compensation upon impact detection as a baseline.}

\subsubsection{\note{Results}}
\begin{figure}[tb]
	\vspace{1mm}
	\centering
	\includegraphics[width=0.45\textwidth]{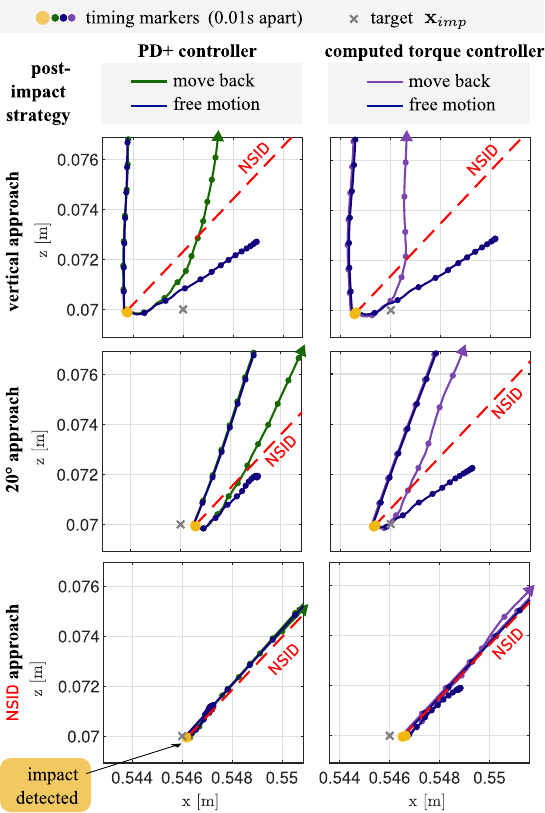}
	\caption{\note{Translational end-effector paths obtained for impacts with the FR3. For impacts under the NSID, the rebound paths align the approach path, also directly after the impact. Conversely, the end-effector is deflected from the path by non-NSID impacts. This immediate deviation cannot be prevented by the controller, but must be recovered over time.}}
	\label{fig:franka_res}
\end{figure}
\note{Figure~\ref{fig:franka_res} shows the translational end-effector paths around the contact for all twelve experiments. The approaches on the left are performed under the PD+ controller. Results for the CT controller are shown on the right.} 

\note{First, consider the top plots belonging to the vertical approach. For the "\textit{free motion}" post-impact strategy, the end-effector is deflected from the vertical approach path in positive $x$-direction. It comes to a full stop approximately $\SI{3}{\milli\metre}$ above the surface. As opposed to these baseline experiments the end-effector motion of the "\textit{move back}" experiments is controlled post-impact. However, directly after the impact the rebound paths do not align with the vertical approach paths, but with the \textit{"free motion"} paths. Thus, the impact introduces the same initial deflection for the controlled robot under two different control approaches and the uncontrolled robot in gravity compensation mode. As expected, the controllers compensate the occured error over time. The depicted section of Fig.~\ref{fig:franka_res} shows the initiated motion towards the desired vertical path.
}
\note{For the approach under the second non-NSID approach at $\alpha = 20^\circ $, the same observations hold as for the vertical approach. }

\note{Conversely, consider the NSID approach under $\alpha = 43.5^\circ $ as depicted in the bottom plots of Fig.~\ref{fig:franka_res}. The rebound path in "\textit{free motion}" aligns the approach path obtained under both control approaches, before the endeffector stops. Moreover, also for the "\textit{move back}" strategy, no significant deviation from the desired path is introduced by the impact.}

\subsubsection{\note{Discussion}}\label{sec:franka:d_paths}\note{The results confirm that an NSID approach aligns the rebound path with the approach path. In the considered experiments with a torque-controlled robot, this was the case under a CT controller, a Cartesian compliance (PD+) controller, and in free motion. Moreover, non-NSID approaches yielded a deviation from the approach paths independent from the rigid-body controller. 
Thus, overall the results suggest that the posed hthesis is valid.}

\note{It is notable that the desired impact target was not perfectly met in all experiments. For the four vertical impacts, we observe a maximum deviation of approximately \SI{2}{\milli\metre} in the $x$-direction, which also implies slight variations in the impact configurations. Figure~\ref{fig:franka_uncertainties}(a) analyzes the sensitivity of the NSID to such configuration uncertainties. In the considered setting, angular deviations of up to $1^\circ$ result in translational end-effector shifts of roughly $\pm\SI{15}{\milli\metre}$, thus covering the observed variations. Yet, they induce only minor deviations in the NSID angle with offsets within $[-1.3^\circ, 1.1^\circ]$. Thus, the small errors in the impact configuration do not affect the NSID considerably and the experiments remain comparable. Figure~\ref{fig:franka_uncertainties}(b) and (c) further confirms that the NSID is not considerably sensitive with respect to expectable uncertainties in both the inertial parameters and the constraint orientation.
}
\begin{figure}[tb]
	\vspace{1mm}
	\centering
	\includegraphics[width=0.49\textwidth]{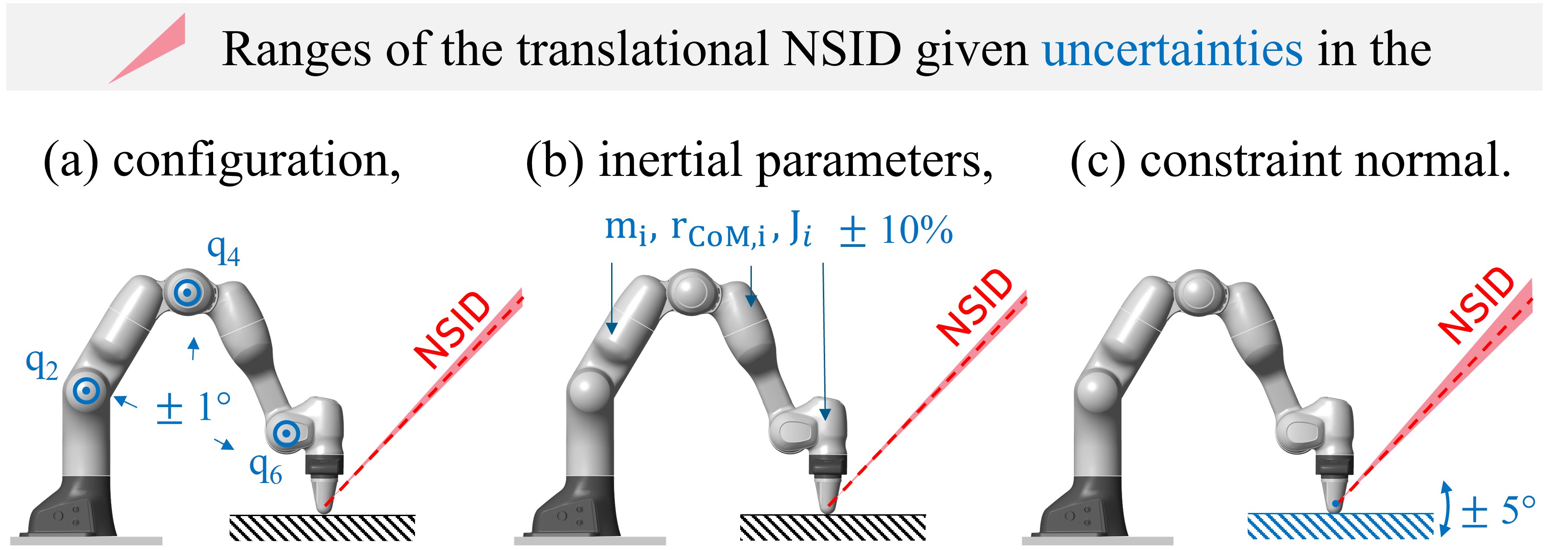}
	\caption{\note{Sensitivity of the translational NSID with respect to uncertainties evaluated for the target configuration in the experiments. Only slight deviations are introduced by expectable mismatches in configuration (a), inertial parameters (b), and constraint normal (c).}}
	\label{fig:franka_uncertainties}
\end{figure}

\note{The experiments were conducted at relatively low approach velocities of $\dot\varphi = -\SI{0.1}{\metre/\second}$ to avoid damaging the torque sensors and gears of the robot. It is plausible that induced errors would increase at higher approach velocities using a system designed to withstand stronger impacts.}

\subsection{\note{Relation Between the Approach Direction and Total Momentum Transfer}}\label{sec:franka_FTS}
\note{Our analysis in Sec.~\ref{sec:friction} revealed that post-impact slippage on a frictional, fully inelastic contact surface can be prevented by choosing the approach directions from a friction-dependent cone. Among those, the NSID is the unique direction for which the impulsive force becomes orthogonal to the contact. In the following, we aim to validate this conclusion in a realistic setting, where several of the underlying assumptions are not perfectly met. First, the previous experiments show a slight post-impact bouncing of the endeffector in free motion, which indicates slight impact elasticity. Moreover, the impact is not instantaneous. Hence, contact forces cannot become infinitely large and momentum is transferred over a small but finite time interval. In that sense, we hypothesize:}
\note{\textit{For an impact under the task space NSID, the momentum change is transferred along the constraint normal at the contact point.}}

\note{We will evaluate this utilizing the FTS readings at the impact for the two tools with different frictional properties.}

\subsubsection{\note{Experiments}}\note{Impact experiments are conducted with both tools using the CT controller pre- and the “\textit{free-motion}’’ strategy post-impact. The approach angle $ \alpha $ is varied from $ -40^\circ $ to $ 60^\circ $. We use a finer resolution in the vicinity of the NSID angle $ \alpha \approx 43.5^\circ $ and larger increments at greater deviations. From the FTS readings, we obtain the forces $ F_x $ and $ F_z $ applied on the endeffector in $ x $- and $ z $-directions, respectively, at $ \SI{2}{\kilo\hertz} $. For the total momentum transferred on the end-effector during the contact in direction $ x $, it holds \begin{equation}\label{eq:momentum_x}
		p_z = \int_{t_0}^{t_e} F_z(t) \; dt,
	\end{equation} where the times $ t_0 $ and $ t_e $ correspond to the beginning and the end of the contact, respectively.\footnote{\note{Note that, for an idealized, instantaneous impact with $ t_e -t_0 \rightarrow 0 $, the momentum $ p_z $ corresponds to the impulsive force $ \Lambda $.}} The momentum $ p_x $ in $ x $-direction is obtained from $ F_x(t) $ accordingly.
	During post-processing, the total momentum transferred during the contact in each direction was computed numerically for every trial. }

\subsubsection{\note{Results}}\note{Figure~\ref{fig:franka_FTS} shows the fraction $ p_x/p_z $ over the approach angle for both tools. Moreover, the underlying forces $ F_x(t) $ and $ F_z(t) $ are provided over time for two characteristic approach angles.}
\begin{figure}[tb]
	\vspace{1mm}
	\centering
	\includegraphics[width=0.46\textwidth]{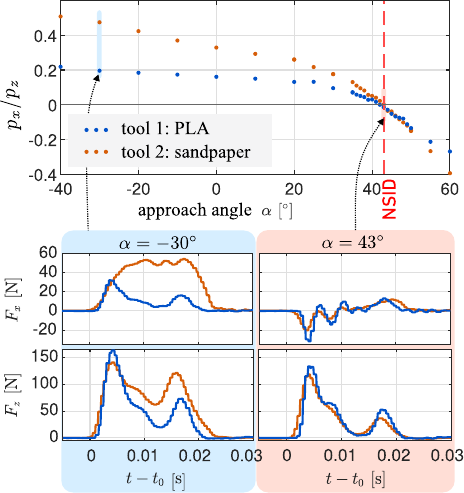}
	\caption{\note{Ratio of momentum transfer in $ x $- and $ z $-direction at the end-effector as calculated from the FTS measurments $ F_x $ and $ F_z $ for impacts under various approach directions. An NSID impact corresponds to vanishing momentum transfer in $ x $-direction for tools with distinct frictional properties.}}
	\label{fig:franka_FTS}
\end{figure}
\note{The absolute values of $ p_x/p_z $ are larger for the sandpaper tool as compared to the PLA tool. This is as expected, as the sandpaper increases the friction between the endeffector and the plywood. Thus, larger forces in $ x $- direction can be transferred given the same normal force. This becomes also visible in the force plot $ F_x(t) $ for $ \alpha = -30^\circ $. There, a force of approximately $ \SI{50}{\newton} $ is transmitted over about $ \SI{0.01}{\second} $ for the impact with the sandpaper tool. Conversely, the force only shows a short peak of up to $ \SI{35}{\newton} $ for the PLA tool. For both tools, the momentum ratio is decreasing with increasing approach angle. The curves intersect approximately at their zero crossing. This approximately corresponds to an approach angle which equals the NSID angle.}

\note{While $ p_x $ vanishes for impacts under the NSID, there are still forces transmitted in $ x $-direction. The corresponding force profile shows oscillatory forces around $ F_x = 0 $, which integrate to zero. Notably, the profile of $ F_x $ is very similar for both tools, as opposed to the approach at $ -30^\circ $. Thus, the differing contact friction seems to have limited effects under the NSID.}

\subsubsection{\note{Discussion}}\note{Our theory predicts that only under an NSID approach of a frictional, inelastic surface, the \textit{impulsive force} becomes perpendicular to the surface under the instantaneous impact. The provided experiments show that in the considered scenario of a slightly elastic, finite-time impact, the same conclusion holds for the \textit{momentum transferred over the finite impact time}. Thus, the posed hypothesis applies in the considered case. This can be an indication that \Cref{thm:perpendicular force} is also meaningful for not perfectly inelastic contacts.}

\note{Future work can focus on analyzing the general case of frictional, elastic impacts. To the best of the authors’ knowledge, suitable instantaneous impact models are still under debate \cite{Gilardi2002, Jia2017}. Consequently, a systematic experimental validation of such models within a robotic context appears necessary.}

\section{Conclusion and Outlook}\label{sec: discussion}
\note{This paper derived the NSID in configuration and task space from a projection-based analysis of the common instantaneous impact model. We performed a comprehensive theoretical analysis of this quantity, highlighting that it is characteristic to a robotic impact scenario. We provided interpretations and covered broad classes of robotic systems and contact properties.}

\note{While the analysis relies on several idealizing assumptions, our experiments show that the results are relevant in practice. In particular, we verified that the approach direction w.r.t. the NSID affects i) the post-impact rebound direction and ii) the direction of momentum transfer at the contact for a frictional, inelastic surface. The NSID was shown to be independent from the contact's restitution factor. Moreover, we proved that it is an inherent property of passive multi-body systems, which prevails for actuated systems and cannot be shaped by rigid body controllers.}

\note{Thus, overall, we are confident that our findings can support future impact-aware planning and control approaches. Impacts performed along the NSID rebound in the same direction, independent from the restitution factor. This ensures that, for example, an impact tool will not damage the specimen after an impact. Rather, it will follow its approach path backward until the controller affects its motion again. In a hammer-and-nail example, moreover, the transfer of momentum strictly along the contact normal prevents the nail from bending. Furthermore, in the inelastic case, NSID impacts prevent slippage, which could, e.g., benefit locomotion.}

\note{Moreover, the presented analysis provides intuitive insights in robotic impact effects, in general. This can inform the design of model-based impact strategies beyond the discussed applications benefitting from an NSID impact. Because our analysis covers a wide range of systems, including fixed-base floating base and constrained systems, torque-controlled robots and impact-robust elastic robots, we believe that the results are accessible to a broad range of domains.}

\note{Future works can also focus on analyzing the NSID for frictional (partially) elastic impacts or systems under multiple inequality constraints.}